\documentclass[dvipsnames]{siamonline220329}

\newsiamremark{remark}{Remark}

\usepackage{graphicx,wrapfig,tabularx,amsmath,amssymb,color} 
\usepackage{xcolor}    
\usepackage{xspace} 
\usepackage[caption=false]{subfig}
\usepackage{enumitem}   
\usepackage{algpseudocode}  
\usepackage{algpseudocodex}
\usepackage{bm} 
\usepackage{bbm} 
\usepackage[export]{adjustbox}
\usepackage{booktabs}
\usepackage{cleveref} 
\usepackage{tikz}
\usetikzlibrary{calc}
\usepackage{standalone} 
\usepackage{mathtools}
\usepackage{wrapfig}
\usepackage{xcolor}
\usepackage{comment}
\let\implies\Rightarrow

\newcommand{\R}{{\mathbb{R}}}

\newcommand{\N}{{\mathbb{N}}}

\newcommand{\Id}{\ensuremath{\mathrm{Id} }}
\newcommand{\dd}{\ensuremath{\mathrm d}}

\newcommand{\cA}{\ensuremath{\mathcal{A}}}

\newcommand{\cG}{\ensuremath{\mathcal{G}}}
\newcommand{\cH}{\ensuremath{\mathcal{H}}}

\newcommand{\cK}{\ensuremath{\mathcal{K}}}
\newcommand{\cL}{\ensuremath{\mathcal{L}}}
\newcommand{\cM}{\ensuremath{\mathcal{M}}}

\newcommand{\cP}{\ensuremath{\mathcal{P}}}

\newcommand{\cW}{\ensuremath{\mathcal{W}}}
\newcommand{\cX}{\ensuremath{\mathcal{X}}}

\newcommand{\bN}{\ensuremath{\mathbb{N}}}

\newcommand{\bR}{\ensuremath{\mathbb{R}}}

\newcommand{\rP}{\ensuremath{\mathrm{P}}}
\newcommand{\rY}{\ensuremath{\mathrm{Y}}}

\DeclareMathOperator*{\argmax}{arg\,max}
\DeclareMathOperator*{\argmin}{arg\,min}

\newcommand{\e}{\ensuremath{\varepsilon}}

\newcommand{\cond}{\;:\;}

\newcommand{\bary}{\ensuremath{\mathrm{Bar}}}

\newcommand{\SE}{\text{SE(2)}}

\newcommand{\WR}{\text{\bf WR }}
\newcommand{\SER}{\text{\bf WG }}
\newcommand{\LR}{\text{\bf L2R }}

\newcommand{\mc}[1]{w_{#1}}

\newcommand{\prox}{\text{prox}}

\let\epsilon\varepsilon

\newcommand{\Z}{{\mathbb{Z}}}

\DeclareMathOperator{\KL}{KL}
\newcommand{\haar}{\cH}
\newcommand{\dhaar}[1]{\rd\cH(#1)}

\newcommand{\rd}{\mathrm{{d}}}

\newcommand{\spt}{\text{{spt}}}

\let\act\triangleright
\newcommand{\norm}[1]{\left\lVert#1\right\rVert}

\title{Optimal Transport on the Lie Group of Roto-translations}

\author{Daan Bon\thanks{Centre for Analysis, Scientific Computing, and Applications, Eindhoven University of Technology (TU/e). {\bf email:} \{d.l.j.bon, g.pai, g.bellaard, o.mula, r.duits \}@tue.nl.} 
\and Gautam Pai\footnotemark[1] \thanks{The first two authors contributed equally.}
\and Gijs Bellaard\footnotemark[1]
\and Olga Mula\footnotemark[1]
\and Remco Duits\footnotemark[1]
}

\headers{Optimal Transport on the Lie Group of Roto-Translations}{D. Bon, G. Pai, G. Bellaard, O. Mula and R. Duits}
\begin{document}
\maketitle

\begin{abstract}
The roto-translation group $\SE$ has been of active interest in image analysis due to methods that lift the image data to multi-orientation representations defined on this Lie group. This has led to impactful applications of crossing-preserving flows for image de-noising, geodesic tracking, and roto-translation equivariant deep learning. 
In this paper, we develop a computational framework for optimal transportation over Lie groups, with a special focus on $\SE$.
We make several theoretical contributions (generalizable to matrix Lie groups) such as the non-optimality of group actions as transport maps, invariance and equivariance of optimal transport, and the quality of the entropic-regularized optimal transport plan using geodesic distance approximations.
We develop a Sinkhorn-like algorithm that can be efficiently implemented using fast and accurate distance approximations of the Lie group and GPU-friendly group convolutions. We report valuable advancements in the experiments on 1) image barycentric interpolation, 2) interpolation of planar orientation fields, and 3) Wasserstein gradient flows on $\SE$. We observe that our framework of lifting images to $\SE$ and optimal transport with left-invariant anisotropic metrics leads to equivariant transport along dominant contours and salient line structures in the image. This yields sharper and more meaningful interpolations compared to their counterparts on $\mathbb{R}^2$. 
\end{abstract}



\section{Introduction}
\subsection{The Roto-translation Group} 
Lie groups provide a potent mathematical framework for understanding geometric transformations and symmetries that frequently arise in many real-world applications.  
Lie group theory provides a natural foundation for expressing these symmetries and enhances the efficiency and interpretability of computational models that incorporate their mathematical structure. 
The Lie group $\SE$ has been of particular interest in image processing. 
Firstly, it is natural to demand that any operation applied to the image (e.g. denoising, segmentation, feature extraction, etc.) has to be \emph{equivariant} to a roto-translation of the image function. 
Simply put, operating on the image after a roto-translation or roto-translating the image after the operation must yield an identical output. 
Secondly, in many applications, it is desirable to identify one or more planar orientations at each 2D location in the image. 
This enables the effective processing of important image sub-structures with directional attributes like lines, edges, crossings, bifurcations, etc. 
For example, in the domain of medical image processing, there is often a requirement of \emph{crossing preserving} de-noising of images. 
Many of these images contain line structures having multiple crossings and bifurcations and we require that the denoising method \emph{preserve} these structures \cite{citti2006cortical, FrankenPhDThesis,ghimpecteanu2015decomposition,  boscain2018highly}. 
Methods not accounting for such local orientation information often lack this property leading to distortionary results. 
Furthermore, medically significant line structures need to be spatially tracked leading to accurate measurements, which is useful for many further diagnoses.  
In such situations, the tracking algorithm must correctly identify and then move along the most cost-effective local orientation in order to extract the appropriate segment in that application. 
\begin{wrapfigure}{r}{0.5\textwidth}
\vspace{-1.3\intextsep}
\centering
\includegraphics[width=\linewidth]{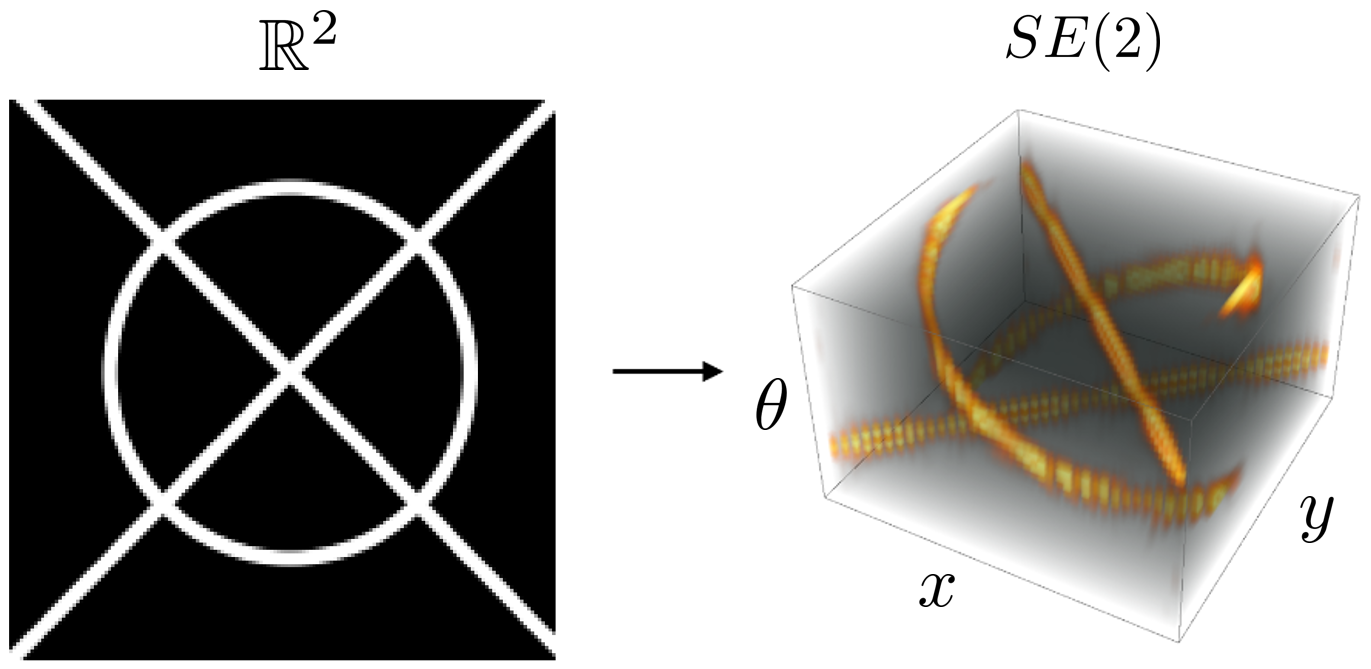}
\caption{An orientation score transform (see \cref{eq:lift}) disentangles line structures due to \emph{lifting} onto the Lie group $\SE$. This is a useful property in many applications like image de-noising, geodesic tracking, equivariant deep learning, etc. In this paper, we use it for the barycentric interpolation of images using optimal transport as seen in the schematic in \cref{fig:intro}} 
\label{fig:lifting_intro}
\vspace{-1\intextsep}
\end{wrapfigure}
This becomes especially critical at junctions that may yield multiple orientations and tracking on the image domain in $\mathbb{R}^2$ is typically insufficient to make the correct choice \cite{benmansour2011tubular, bekkers2015pde, vandenbergJMIV2023}. 
\begin{figure}[t]
\centering
\includegraphics[width=1\linewidth]{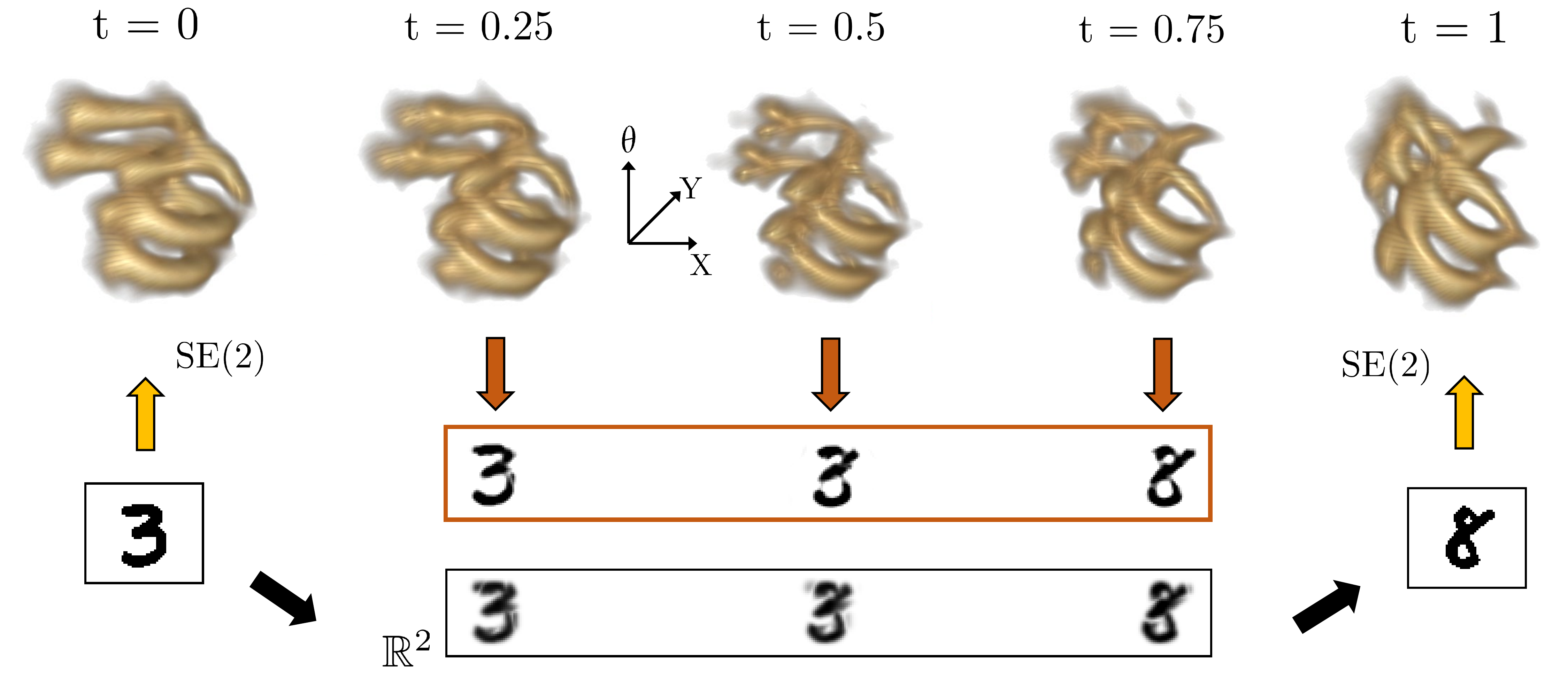}
\caption{
\textbf{Barycenter interpolation in $\cP(\SE)$:} To interpolate classical 2D images in $L^2(\bR^2)$ (in this case, an image of 3 and an 8), we first lift them to the space of probability measures $\cP(\SE)$ with operations involving  invertible orientation scores (see \cref{fig:lifting_intro} and equation \cref{eq:lift}). We then compute Wasserstein barycenters on that space (see \cref{equ:interpolation_operator}) using linear weights indicated by $t$. Finally, we project the output back to $L^2(\bR^2)$ (by integration over $\theta$, cf.~\cref{eq:proj}). We observe sharper and more meaningful interpolation in this process due to the left-invariant anisotropic metrics in $\SE$ that yield a unique contour-concentrating behavior of the optimal transport. This is in contrast to isotropic transport in $\R^2$ where one observes significant mass splitting and leakage away from the main contours.} 
\label{fig:intro}
\end{figure}

A common feature of orientation-aware methods that efficiently address these issues is that they \emph{lift} the image data into the homogeneous space of positions and orientations. The lifting is usually done through an orientation score transform as we depict in \cref{fig:lifting_intro} (see also \cite{citti2006cortical, duits2005perceptual, bekkers2017retinal, boscain2018highly, chambolle2019total}). 
The Lie group of roto-translations $\SE: = \mathbb{R}^2 \rtimes \text{SO}(2)$ acts naturally on this space of positions and orientations \mbox{$\mathbb{M}^2: = \mathbb{R}^2 \times S^1$}, which is also its principal homogeneous space. In fact, by fixing a reference position-orientation in \(\mathbb{M}^2\) the two spaces can be identified with each other. 
The geometric algorithm is then applied in this lifted space, and finally, the processed orientation score is projected back yielding the output. 
Methods that follow this broad workflow are equivariant to roto-translations of the input by design. 
This workflow is also the backbone behind many successful roto-translation equivariant deep-learning architectures (called G-CNNs): \cite{sifre2013rotation, bekkers2018roto, weiler2019general} that are more efficient and need less training data. 
The computational machinery in all these methods indispensably use differential geometric structures like left-invariant Riemannian and sub-Riemannian geodesics, distances and kernels of the Lie group $\SE$. 

Despite these impactful applications, there has been little progress on the theory and applications of optimal transport problems in this domain. 
We report a relevant related work \cite{galeotti2022cortically} that explores image morphing in the visual cortex using optimal transport. 
Functions in $\mathbb{R}^2$ are lifted onto this cortical space using Gabor filters parameterized by position, orientation, and scale. However, in our article, we uniquely exploit the $\SE$ group structure (analytic distance approximations and group convolutions) to provide a scalable (Sinkhorn-like) algorithm. We apply our algorithm to orientation scores (constructed by 'all-scale' cake-wavelets), multitude of images, orientation fields, and to compute numerical solutions to PDEs. Finally, we provide new theoretical results as outlined in \cref{sec:contributions}. 

\subsection{Optimal Transport on Geometric Domains}
Optimal transport (OT) provides a principled and versatile approach to work with measures defined on very general spaces including Riemannian manifolds and $\SE$. Among the many ideas coming from OT, we may mention mathematical tools to compare or interpolate between probability distributions.
The prototypical example of an OT problem is the two-marginal Monge-Kantorovich assignment: given  two Polish spaces $X_1$ and $X_2$ and given two probability measures $\mu\in \cP(X_1)$ and $\nu\in \cP(X_2)$, solve 
\begin{equation}
\label{eq:MK}
\inf \biggl\{ \int_{X_1\times X_2} c(x_1,x_2) \rd\pi(x_1,x_2) \cond \pi \in \Pi(X_1 \times X_2; \mu , \nu) \biggr\}
\end{equation}
where $c:X_1\times X_2\mapsto \bR^+$ is a lower semi-continuous cost function. The infimum runs over the set $ \Pi(X_1 \times X_2; \mu , \nu) $ of coupling measures (usually called transport plans) on $X_1\times X_2$ with marginal distributions equal to $\mu$ and $\nu$ respectively. More generally, the problem can be multi-marginal as we introduce later on in \cref{sec:OT}.

Starting with the seminal formulations by Monge \cite{monge1781memoire}, Kantarovich \cite{kantorovich1942On} and later on by Brenier \cite{brenier2003extended}, the theoretical discourse on optimal transport is very extensive \cite{villani2003topics, villani2008optimal, santambrogio2015optimal}. It has been shown to have rich and diverse connections to the theory of Partial Differential Equations (gradient flows \cite{ambrosio2005gradient}, fluid mechanics \cite{Benamou1999computational}, quantum chemistry \cite{CFP2015, BCN2016}), probability theory \cite{ambrosio2005gradient}, and geometric analysis \cite{mccann2001polar, lott2006some} just to name a few. Many important PDEs like linear and non-linear diffusion and advection equations can be viewed as gradient flows on Wasserstein measure spaces. Such PDEs have a neurogeometrical significance when they are solved on $\SE$ \cite{petitot2003neurogeometry, citti2006cortical,FRANCESCHIELLO201955} and also have practical impact when combined with neural architectures \cite{smets2023pde} in deep learning.

In recent years, the development of efficient numerical approaches has opened the door to applying OT to numerous applications, notably in computer vision, image processing, computer graphics, and machine learning \cite{ferradans2014regularized, delon2020wasserstein, mandad2017variance, merigot2018algorithm, lacombe2018large, alvarez2020geometric}. Among the many numerical strategies that have been proposed in the literature, probably the most popular, and prominent direction consists in adding an entropic regularization term to the problem \cite{cuturi2013sinkhorn, benamou2015iterative, solomon2015convolutional, feydy2019interpolating}, and considering an approximation of the measures by discrete measures via sampling or discretization on a grid. The discretization, and the additional entropy term lead to an approximation of the transport plan that can be computed using a fixed-point strategy, popularly known as the Sinkhorn algorithm. The Sinkhorn iterations have a linear convergence and only require convolutions with Gibbs kernels, thereby computing a smoothed version of the transport plan. For discretized geometric domains like triangulated meshes, point clouds, and graphs, a particularly widespread solution is the use of heat kernels \cite{solomon2015convolutional, varadhan1967behavior} in lieu of the distance-based Gibbs kernel. This avoids the need for computing distances and only evolving the heat equation which is in many applications numerically more straightforward on these domains. 
These advancements have been enormously beneficial in many applications: barycenters and shape interpolation \cite{cuturi2014fast, courty2018learning}, generative models for machine learning \cite{onken2020OTflow, arjovsky2017wasserstein, houdard2023generative}, 3D shape registration and correspondence \cite{eisenberger2020deep, shen2021accurate, pai2021fast} to name a few.  We refer to some excellent in-depth surveys in \cite{papadakis2015optimal, peyre2020computational, bonneel2023survey} for a more exhaustive list of the diverse applications of optimal transport. More closely related to the theme of this paper, optimal transport has also been studied for the interpolation of vector and tensor valued measures in \cite{solomon2019optimal, peyre2019quantum} as well as the rotation group SO(3) \cite{quellmalz2024parallelly, birdal2020synchronizing}. As we demonstrate later in \cref{sec:vector-fields}, orientation fields in $\mathbb{R}^2$ can be naturally lifted onto the homogeneous space of the Lie group $\SE$ where we can employ anisotropic metrics for meaningful interpolation.

\subsection{Contributions}
\label{sec:contributions}
In this paper, we formulate various optimal transport problems like the Wasserstein distance, barycenters, and gradient flows on the roto-translation group $\SE$ by operating on its principal homogeneous space $\mathbb{M}^2 = \mathbb{R}^2 \times S^1$. In fact, most of our theoretical results can be generalized to any finite-dimensional Lie group, but we focus on numerical experiments involving $\SE$ by lifting data from $\mathbb{R}^2$ to $\mathbb{M}^2$ (see \cref{fig:intro}). Our main contributions in this paper can be enumerated as follows:

\begin{enumerate}
\item \textbf{Entropic OT on Lie groups:} We extend the main result of \cite{solomon2015convolutional} by demonstrating that entropic optimal transport on Lie groups can be solved with the Sinkhorn algorithm involving only \emph{group-convolutions} with left-invariant kernels on the input distributions. We show that left-invariant costs lead to the invariance of the Wasserstein distance and equivariance of barycenters and gradient flow iterations for the same group actions on all input arguments of these applications.        
\item \textbf{Efficient approximation of the OT cost function $c$:} The numerical cost of the Sinkhorn algorithm quickly becomes prohibitive when the number of discretization points increases, and when the evaluation of the cost function at each of the mesh points is not negligible. For OT problems formulated on general Riemannian manifolds like $\SE$, the cost function $c$ is usually the geodesic distance between points of the discretization mesh in $\SE$. Since it has no closed form,  one needs to rely on fast approximations to solve the OT problem. In this paper, we address this issue by leveraging the logarithmic coordinates of $\SE$ \cite{bellaard2022analysis} to obtain a fast and accurate approximation of the Gibbs kernel. We make three experimental demonstrations:
\begin{enumerate}
\item computing barycenters by lifting images to $\SE$ (the conceptual flow is summarized in \cref{fig:intro}),
\item interpolations of 2D orientation fields, and
\item solving PDEs on $\SE$ using Wasserstein gradient flows.
\end{enumerate}
Our experiments reveal favorable interpolations, aided by the lifting and anisotropic metrics that promote transport along salient line structures of the image. This helps reduce significant mass-splitting and promotes sharper results in the interpolations.
\item \textbf{Theoretical results:} 
\begin{enumerate}
    \item We establish that for a wide class of Lie Groups, including $\SE$, the right (and left) group action is generally \underline{not} equal to the optimal transport map between a final distribution that is transformed by a group action of an initial distribution. However, such an equality statement does hold for Euclidean groups like $(\mathbb{R}^n, +)$ and also for certain sub-Riemannian metrics and boundary-conditions on the Heisenberg group \cite[Example 5.7]{ambrosio2004optimal}.
    \item We derive corresponding bounds between Wasserstein distances constructed with the exact Riemannian distance and our distance approximations on the Lie group and also derive a similar result for the minimizing regularized optimal transport plans. This provides a theoretical justification for the use of the proposed approximations in practice. 
\end{enumerate}
\end{enumerate}

\subsection{Structure of the Article}
In \cref{sec:lie_groups}, we begin by giving a background on Lie groups, optimal transport, and details about our main numerical example $\SE$. In \cref{sec:symm_ot} we explore the symmetries of optimal transport on Lie groups and additional results on the optimality of these symmetries. In \cref{sec:ER_scaling} we formulate the entropic regularization for the OT problem in $\SE$, and build a Sinkhorn algorithm, based on group convolutions and a distance approximation that can be efficiently implemented in this domain. We also show that error bounds on the approximation give rise to error bounds on the OT problems. Finally, in \cref{sec:numerical_results} we report the results of numerical experiments on images, orientation fields, and gradient flows on $\SE$. 
\section{Background: Optimal Transport and Lie Groups}
\label{sec:lie_groups}

\subsection{Lie Groups}
\label{sec:liegroup_riemannian}
We recall that a Lie group is a smooth manifold $G$ equipped with a group structure such that the group multiplication $\texttt{mult}: G\times G\to G$ given by $\texttt{mult}(g,h)=gh$ and the inverse $\texttt{inv}: G\to G$ given by $\texttt{inv}(g)=g^{-1}$ are smooth maps. The operations of left and right translation by a fixed element $g\in G$ are defined as the mappings
\begin{equation}
\label{def:left_right_translation_group}
L_g : G \rightarrow G, \;\; L_g(h) = g h \;\;\;\text{ and } \;\;\; R_g : G \rightarrow G, \;\; R_g(h) = hg.
\end{equation}
Every $g \in G$ has an associated tangent space $T_g G$ which is the linear space spanned by all tangent vectors of curves passing through $g$. This tangent space has a dual space, denoted by $T^*_g G$, called the cotangent space. The tangent space at the identity element $e\in G$, $T_e G$, is called the Lie algebra. 
For a smooth map $\varphi:G\to N$ between $G$ and another smooth manifold $N$, the differential of $\varphi$ at $p\in G$ is a linear map
$
(\dd\varphi)_p : T_p G \to T_{\varphi(p)}N
$
from the tangent space of $G$ at $p$ to the tangent space $N$ at $\varphi(p)$. The image of $\dd\varphi_p(\dot p)$ of a tangent vector $\dot p \in T_p G$ under $\dd\varphi_p$ is called the push-forward of $\dot p$ by $\varphi$. 
In the case $\varphi=L_g$ for a given $g\in G$, note that
$$
(\dd L_g)_p : T_p G \to T_{gp}G.
$$

\subsubsection*{Group Action, Invariance, and Equivariance}

Given a group \(G\) and a set \(X\), a \textit{group action} \(A_g : X \to X\) is a mapping that satisfies \(A_{g_2} \circ A_{g_1} = A_{g_2 g_1}\) for all $g_1,g_2 \in G$ and \(A_{e} = \text{id}_X\). We say that \(G\) \textit{acts} on \(X\). 
Let \(A_g : X \to X\) and \(B_g : Y \to Y\) be two group actions. 
A function \(\varphi : X \to Y\) is \textit{invariant} (under $A_g$) if \(\varphi \circ A_g = \varphi\).
A function \(\varphi : X \to Y\) is \textit{equivariant} (w.r.t. $A_g$ and $B_g$) if \(\varphi \circ A_g = B_g \circ \varphi\).
In the specific case where the group action is the left-translation \(L_g\) (or some corresponding derived notion) we speak of \textit{left-invariance}, analogously, when the group action is the right-translation \(R_g\) we speak of \textit{right-invariance}.

\subsubsection*{Left Invariant Metric Tensor and Distance}

As an additional structure, we endow the group $G$ with a metric tensor field $\cG$, turning it into a Riemannian manifold which we denote by $(G, \cG)$. The metric tensor at a point $p\in G$ is a bilinear symmetric positive definite map $\cG_p: T_pG \times T_pG \to \R$ which defines an inner product in the linear space $T_pG$. The mapping $p\mapsto \cG_p$ is a smooth function of $p$. The metric tensor induces a natural way of defining a distance \(d_\cG\) on $G$ as
\begin{equation}
\label{eq:geodesic-G}
    d_\cG(p, h) := \inf_{ \substack{\gamma \in PC([0, 1]; G) \\ \gamma(0)=p,\, \gamma(1)=h} } L(\gamma), \quad \text{ where } L(\gamma) := \int_0^1 \sqrt{ \cG_{\gamma(t)} (\dot{\gamma}(t), \dot{\gamma}(t))} \rd t,
\end{equation}
where $\text{PC}([0,1],G)$ denotes the family of piece-wise continuously differentiable curves in $G$. We omit the subscript from $d_\cG$ if it is clear from the context what metric is considered. 

We say that the metric tensor is left-invariant if and only if for all \(g,h \in G\) and \(u,v \in T_h G\):
\begin{align} 
\label{eq:left-inv-metric-tensor}
\cG_h (u, v)
=
\cG_{L_g(h)}( (\dd L_g)_{h} u, (\dd L_g)_{h} v).
\end{align}
When this property holds, it induces a left-invariant distance
\begin{equation}
\label{eq:left-inv-distance}
d_\cG (p, h) = d_\cG ( L_g(p), L_g(h) ), \quad \forall g, p, h \in G.
\end{equation}
Left-invariant metrics play a key role because they allow for designing equivariant image processing operations.

\subsection{Optimal Transport}
\label{sec:OT}
Next, we introduce optimal transport problems posed on a Lie group $G$ with a left-invariant Riemannian metric $d$. We define the most classical problems such as the one of the Wasserstein distance, barycenters and gradient flows, and we discuss certain properties. Other common problems such as Gromov-Wasserstein discrepancies and barycenters could be introduced but we omit them in our presentation for the sake of brevity. 

Let $\cM(G)$ be the space of finite signed measures on $G$. Its positive cone and the set of probability measures are defined as
$$
\cM(G)_+ \coloneqq \{ \mu \in \cM(G)\cond \mu\geq 0 \}, \;\; \cP(G)\coloneqq \{ \mu \in \cM_+(G) \cond \mu(G)=1 \}.
$$
The OT problems that we consider in this work can be expressed as follows. Let $n\in \bN$, and let $\cL : \cP(\cX) \rightarrow \R$ be a loss function with $\cX = G^n = G \times \cdots \times G$. Given $K$ measures $\mu_1, \dots, \mu_K \in \cP(G)$ with $1\leq K\leq n$, the task is to solve
\begin{equation}
\label{eq:MK-multimarg}
    \rho \coloneqq \inf_{\pi \in \Pi} \cL(\pi), 
\end{equation}
where
\begin{equation}
\label{eq:couplings}
    \Pi = \Pi(\cX;\mu_1,\dots, \mu_K) 
    \coloneqq
    \{ \pi \in \cP(\cX) \cond  (\mathrm{proj}_i)_\# \pi  = \mu_i,\quad 1\leq i \leq K \},
\end{equation}
is the set of coupling measures having the collection of $\mu_i$'s as marginals. In the above formula, $\mathrm{proj}_i: \cX \to G$ denotes the projection on the $i$-th component, and the push-forward measure $(\mathrm{proj}_i)_\# \pi$ is the $i$-th marginal of $\pi$. For a measurable map $T: G \rightarrow G$, and a measure $\mu \in \cP(G)$ the push-forward measure $T_\# \mu \in \cP(G)$ is defined by the relation 
\[
    T_\# \mu (A) = \mu(T^{-1}(A))   \text{ for all measurable } A 
    \subseteq G,
\]
where $T^{-1}(A)$ denotes the pre-image of $A$ under $T$. One important example of this will be the push-forward action of left-translations,
\begin{equation}
\label{equ:measure_group_action}
    (L_g)_\# : \cP(G) \rightarrow \cP(G), \textrm{ given by } \  ((L_g)_\# \mu)(A)= \mu(L_g^{-1} A),
\end{equation}
which is a group action of $G$ on $\cP(G)$. The following classical OT problems can be written in the above general form:
\begin{itemize}
\item \textbf{Wasserstein distances:} For $p\geq 1$, the Wasserstein space $\cP_p(G)$ is defined as the set of probability measures $\mu\in \cP(G)$ with finite moments up to order $p$, namely
$$
\cP_p(G) \coloneqq \biggl\{ \mu \in \cP(G) \cond \int_G d^p(x, x_0) \,\rd\mu(x) \;< +\infty \biggr\},
$$
where $x_0\in G$ is a (arbitrary) reference point in $G$ and where we write $d^p(x,y) := (d(x,y))^p$ . Let $\mu$ and $\nu$ be two probability measures in $\cP_p(G)$. The $p$-Wasserstein distance $W_p(\mu,\nu)$ between $\mu$ and $\nu$ is defined by
\begin{equation}
\label{eq:Wp}
W_p(\mu,\nu) \coloneqq \left(\mathop{\inf}_{\pi \in \Pi(G \times G;\mu,\nu)} \int_{G \times G} d^p(x,y) \,\rd\pi(x, y)\right)^{\frac{1}{p}}. 
\end{equation}
The space $\cP_p(G)$ endowed with the distance $W_p$ is a metric space, usually called $L^p(G)$-Wasserstein space (see \cite{villani2003topics} for more details). The $W_p$ distance  \cref{eq:Wp} is an optimal transport problem of the form \cref{eq:MK-multimarg} with $n=K=2$, and a l.s.c.~cost function $c(x,y)= d^p(x, y)$ giving rise to
\begin{align}
\label{eq:LWp}
\cL^{W_p}(\pi) \coloneqq \int_{G \times G} d^p(x, y) \,\rd\pi(x, y).
\end{align}
Optimizers $\pi^*$ of \cref{eq:Wp} are called \textit{optimal transport plans}. Under certain \emph{conditions} on the marginals, see \cite[Corollary 9.3]{villani2008optimal}, one can guarantee the existence of an \textit{optimal map} $T^*: G \to G$ such that $\pi^* = (Id \times T^*)_\# \mu$ \footnote{\((Id \times T^*)(g) = (g,T^*(g))\).}, which implies that $T^*_\# \mu = \nu$.
\item \textbf{Barycenters:} To approximate or interpolate measures, one often resorts to \textit{Wasserstein barycenters}. Let $n \in \bN^*$ and let
$$
\Sigma_n \coloneqq \Big\{ \lambda \in \bR^n\cond \lambda_i\geq 0,\, \sum_{i=1}^n \lambda_i = 1 \Big\}
$$
be the simplex in $\bR^n$. We say that $\bary(\rY_n, \Lambda_n) \in \cP_p(G)$ is a barycenter associated to a given set $\rY_n = (\mu_i)_{1\leq i\leq n}$ of $n$ probability measures from $\cP_p(G)$ and to a given set of weights $\Lambda_n = (\lambda_i)_{1\leq i\leq n} \in \Sigma_n$, if and only if $\bary(\Lambda_n, \rY_n)$ is a solution to
\begin{equation}
\label{eq:barywass}
\bary(\rY_n, \Lambda_n) \in \argmin_{\nu \in \cP_p(G)} \sum_{i=1}^n \lambda_i W_p^p(\nu,\mu_i).
\end{equation}
Problem \cref{eq:barywass} can be written as an optimization problem of the form \cref{eq:MK-multimarg} by setting $\cX = G^{n+1}$, $\Pi = \Pi(\cX; \mu_1, \cdots, \mu_n)$ and defining the cost for $g=(g_1,\dots, g_{n+1}) \in \cX$ as
\begin{equation}
\label{equ:cost_bary}
    c(g_1,\dots, g_{n+1}) = \sum_{i=1}^n \lambda_i d^p(g_i, g_{n+1}),
\end{equation}
which yields a loss function of the form
\begin{align}
\label{equ:cost_barycenter}
\cL^{\bary}(\pi) = \int_\cX c \, \rd \pi
&= \int_{\cX} \sum_{i=1}^n \lambda_i d^p(g_i, g_{n+1}) \rd \pi(g_1,\dots, g_{n+1}) \\
&=  \sum_{i=1}^n \lambda_i \cL^{W_p}\big((\mathrm{proj}_{i, n+1})_\# \pi\big), \quad \forall \pi \in \cP_p(\cX)
\end{align}
where $(\mathrm{proj}_{i, j}) : \cX \rightarrow G \times G$ denotes the projection on the i'th and j'th marginal. Note that $(\mathrm{proj}_{i, n+1})_\# \pi \in \cP(G\times G)$ is a coupling matrix between $\mu_i$ and $(\mathrm{proj}_{n+1})_\# \pi$. In this way one retrieves \cref{eq:barywass}, with $\bary(Y_n, \Lambda_n) = (\mathrm{proj}_{n+1})_\# \pi$. The case $n=2$ deserves some special attention. Let $\Lambda_t = (t, 1-t)$ for $0\leq t\leq 1$, then we define the \textit{interpolation operator} $\Phi_t: \cP_p(G) \times \cP_p(G) \rightarrow \cP_p(G)$ as follows
\begin{equation}
\label{equ:interpolation_operator}
    \Phi_t(\mu, \nu) := \bary\big( (\mu, \nu), \Lambda_t \big).
\end{equation}
This operator is also called the \textit{displacement interpolation} or \textit{McCanns interpolation} between $\mu$ and $\nu$.
\item \textbf{Gradient flows:} Given an initial condition $\mu_0 \in \cP(G)$, an energy functional $F: \cP(G) \rightarrow \R$ and a timestep $\tau > 0$, a discrete gradient flow evolution is defined as the time-marching
\begin{equation}
\label{equ:GF_JKO}    
    \mu_{k+1} \in \argmin_{\nu \in \cP(G)} W_2^2(\mu_k, \nu) + \tau F(\nu), \quad \forall k\in \bN.
\end{equation}
A single iteration is also referred to as a JKO step, see \cite{jordan1996variational}, and can be written in the form of \cref{eq:MK-multimarg} by taking 
\begin{equation}
\label{equ:cost_GF}    
    \cL^{GF}(\pi) = \cL^{W_p}(\pi) + F((\mathrm{proj}_2)_\# \pi),
\end{equation}
and $\Pi = \Pi(\cX ; \mu_k)$. The second marginal of $\pi$ plays the role of $\nu$ in \cref{equ:GF_JKO}.
\end{itemize}

\subsection{The Special Euclidean group}
Our Lie group of interest is the three-dimensional \textit{Special Euclidean group} $\text{SE}(2)$ of all rigid body motions (or roto-translations) of $\R^2$. It is defined as the semi-direct product of $\R^2$ and $\text{SO}(2)$,  $\text{SE}(2) = \R^2 \rtimes \text{SO}(2)$. We write $g=(x, R) \in \text{SE}(2)$, with $x\in \R^2,\, R \in \text{SO}(2)$, and define group multiplication
\begin{equation}
    g_1 g_2 \coloneqq (x_1 + R_1 x_2, R_1 R_2), \quad \forall g_1=(x_1, R_1), g_2=(x_2, R_2) \in \SE.
\end{equation}
and with inversion given by $g^{-1} = (- R^{-1}x, R^{-1})$.
The group $\text{SE}(2)$ acts naturally on $\R^2$: 
\begin{equation}
    \label{equ:group_action_SE2}
        g \act y = (x, R) \act y = R y + x, \textrm{ for all } y \in \R^2, \, g \in \text{SE}(2).
\end{equation}
Here we have used $\act$ as a shorthand notation for the group action. We can identify any element $g \in \SE$ with $(x,y,\theta) \in \R^2 \times \R / (2\pi \Z)$, where we use the small-angle convention that $\R / (2\pi \Z) = [-\pi, \pi)$. We call these coordinates the \textit{fixed coordinate system} on $\SE$. Then for $(x,y,\theta), (x',y',\theta') \in \SE$ the group product becomes
\begin{equation*}
(x, y, \theta) (x', y', \theta') = ( x + x' \cos \theta - y' \sin \theta \, , \, y + x' \sin \theta + y' \cos \theta \, , \, \theta + \theta' \mod 2\pi ),
\end{equation*}
and the identity element $e=(0,0,0)$.
Given the basis $\{A_1, A_2, A_3\} = \{ \partial_x |_{e}, \partial_y|_{e}, \partial_\theta|_{e} \}$ for the Lie algebra $T_e(\SE)$, we get the so called left-invariant vector fields $\mathcal{A}_i |_g = ({\rm d}L_g)_e A_i$ via the pushforward of left translations:
\begin{align*}
    (\mathcal{A}_1)|_{g} &= +\cos(\theta) \left. \partial_x \right|_{g} + \sin(\theta) \left.\partial_y\right|_{g} \\
    (\mathcal{A}_2)|_{g} &= -\sin(\theta) \left. \partial_x \right|_{g} + \cos(\theta) \left. \partial_y \right|_{g} \\
    (\mathcal{A}_3)|_{g} &= \left. \partial_{\theta} \right|_{g}
\end{align*}
for all $g=(x, y, \theta) \in \SE$. The left-invariant Riemannian distances on $\SE$ that we consider\footnote{One could extend to data-driven left-invariant metric tensor fields like in \cite{vandenbergJMIV2023}, but this is beyond the scope of this article.} are induced by left-invariant metric tensor fields characterized by
\begin{equation}
\label{eq:metric_SE2}
    \mathcal{G}_g(\cA_i|_{g},\cA_j|_{g}) = \begin{bmatrix}
        w_1^2 & 0 & 0\\
        0 & w_2^2 & 0\\
        0 & 0 & w_3^2
    \end{bmatrix} 
\end{equation}
where $\mc{i} > 0$ are constant. We also define the \textit{spatial anisotropy} $\zeta$ of the left-invariant metric tensor field $\cG$ as 
\begin{equation}    
\label{equ:spatial_anisotropy}
    \zeta \coloneqq  \frac{\max(\mc{1}, \mc{2})}{\min(\mc{1}, \mc{2})} \in [1,\infty).
\end{equation}
We say the metric is (spatially) isotropic if \(\zeta=1\), and anisotropic otherwise.
Lastly we need the Haar measure on $\SE$, which we will denote by $\haar$. Up to a scalar constant this measure coincides with the Lebesgue measure on $\R^2 \times (\R/(2\pi \mathbb{Z})) \equiv \R^2 \times [-\pi, \pi)$, given by $\dhaar{g} = \rd x\rd y \rd\theta$.

\subsubsection{Lifting and Projecting Images to and from Orientation Scores}
\label{sec:lifting}
In this paper, we focus on numerical examples that lift images $f \in L^2(\R^2)$ to \textit{orientation scores} $U \in L^2(\SE)$. 
We rely on the so-called \textit{orientation score transform} defined as the mapping
\begin{align}
\cW_\psi: L^2(\bR^2) &\to L^2(\SE) \nonumber\\
f&\mapsto U(g)=(\cW_\psi f)(g) \coloneqq \int_{\R^2} \psi( g^{-1} \act x)\; f(x)\ \rd x,\quad \forall g \in \SE,
\label{eq:lift}
\end{align}
where $\psi \in (L^1 \cap L^2)(\R^2)$ is a chosen wavelet function. \Cref{fig:lifting_intro} gives an illustration of the transform. Its usefulness depends on the chosen wavelet $\psi$. It is worth mentioning that the orientation score transform is the only way to lift linearly onto the group $\SE$ \cite[Thm.1]{bekkersbspline2019}. 

In this paper, we choose to employ \textit{cake-wavelets} (for details, we refer to \cite[ch:4.6.1]{duits2005perceptual}). 
This choice results in a stable and interpretable lifting as can be seen in \cref{fig:lifting_intro}. Typically, no information is lost during the lifting from a function to its orientation score \cite{duits2005perceptual}. One can reconstruct disk limited functions from the orientation score by a \textit{projection} which we define as:
\begin{align}
\rP: L^2(\SE) &\to L^2(\bR^2) \nonumber\\
U &\mapsto \rP(U)\coloneqq \int_{-\pi}^{\pi} U(\cdot,\theta)\; {\rm d}\theta. \label{eq:proj}
\end{align}
The choice of the wavelet function $\psi$ (in particular cake-wavelets) directly influences the stability of the invertible orientation score transform $\mathcal{W}_{\psi}$ and enables a meaningful projection \ref{eq:proj} for the reconstruction of disk limited functions. We refer to \cite{duits2007image, duits2010left} for a formal derivation of these properties and \cite{remcolecture} for a comprehensive explanation. We provide a concise summary of these results in Appendix~\ref{app:reconstruct}. In \cref{fig:intro} the lifting operator (orientation score transform) is depicted by a yellow arrow, whereas the projection is indicated by an orange arrow.  

Other possible lifting mechanisms include channel representations \cite{forssen}, tensor voting \cite{medioni}, or Gabor filters \cite{galeotti2022cortically}. Such mechanisms are highly valuable if one wants to process orientation fields $(x,y) \mapsto v(x,y) \in [0, 2\pi)$ (on images) \cite{medioni} or robust orientation estimates \cite{forssen,Felsberg2} rather than the scalar-valued images/densities $f$ themselves. We use one such lifting method for optimal transport over orientation fields in \cref{subsec:lift_of}.

\subsubsection{The Importance of Equivariant Processing} \label{sec:lifting_equivariance_relation}
We highlight an important connection between the lifting of images and the equivariant processing of orientation scores. 
We start by defining the following two natural group actions of $G=\SE$, \(\mathcal{U}_g : L^{2}(\R^2) \to L^{2}(\R^2)\) and \(\mathcal{L}_g : L^{2}(\SE) \to L^{2}(\SE) \):
\begin{equation}
    (\mathcal{U}_g f)(x) := f(g^{-1}\act x), \quad \ 
    (\mathcal{L}_g U)(h) := U(g^{-1}h).
\end{equation}
Using \(\mathcal{U}_{g}\) we can succinctly rewrite the orientation score transform as $(\mathcal{W}_{\psi}f)(g)=(\mathcal{U}_{g}\psi,f)$, where \((\cdot,\cdot)\) is the inner product on \({L^{2}(\R^2)}\).
From the identity \((\mathcal{U}_g f_1, \mathcal{U}_g f_2) = (f_1, f_2)\) 
and the fact that $\mathcal{U}$ is a group representation
one directly deduces:
\begin{equation} \label{eq:fundID}
    \mathcal{W}_{\psi}
    \circ \mathcal{U}_{g} =\mathcal{L}_{g} \circ \mathcal{W}_{\psi} \textrm{  for all }g \in G,
\end{equation}
which means that lifting a roto-translated image is equivalent to left-translating the lift of the original image.

Consider an effective operator \(\Upsilon:L^2(\bR^2)\to L^2(\bR^2)\) on the image domain that lifts images \(f\in L^2(\bR^2)\), does some processing on the orientation score via an operator \(\Phi:L^2(\SE)\to (L^1 \cap L^2)(\SE) \), and projects back:
\begin{equation}
    \Upsilon(f) \coloneqq (\rP\circ \Phi\circ \cW_\psi)(f) =\int_{-\pi}^\pi (\Phi \circ \mathcal{W}_{\psi}f)(\cdot,\theta)\, {\rm d}\theta.
\end{equation}
It is desirable for such an operator \(\Upsilon\) to be \textit{roto-translation equivariant}, that is
\begin{equation}
    \Upsilon \circ \mathcal{U}_{g} = \mathcal{U}_{g} \circ \Upsilon \textrm{ for all } g \in G.
\end{equation}
This holds if the operator $\Phi$ on the orientation score is \textit{left-equivariant}: 
\begin{equation} \label{LIdef}
\Phi \circ \mathcal{L}_g =\mathcal{L}_{g} \circ \Phi \textrm{ for all } g \in G.
\end{equation}
One can quickly show this result using \cref{eq:fundID} (see \cite[Thm.21]{duits2005perceptual} for further details). 

In the following, our approach is to work with probability measures instead of functions at the level of $\SE$. 
We thus build lifting operators $\widetilde \cW_\psi: L^2(\bR^2) \to \cP(\SE)$, work with operators $\widetilde\Phi:\cP(\SE)\to \cP(\SE)$ inspired from OT for the processing, and project back via operators of the form $\widetilde \rP:\cP(\SE)\to L^2(\bR^2)$. 
In a similar manner as for the approach with functions, if we want to guarantee roto-translation equivariance, we need to work with OT operators that are equivariant w.r.t. the group action defined in \cref{equ:measure_group_action}. 
For now we define
\begin{equation}
\label{eq:tildeW-v1}
(
\widetilde{\mathcal{W}}_{\psi}(f))(A) :=
\frac{\int_{A}
(\mathcal{W}_{\psi}f(q))_+ \dhaar{q}}{\int_{G}
(\mathcal{W}_{\psi}f(q))_+ \dhaar{q}}\geq 0
\end{equation}
with $a_+=\max\{a,0\}$ 
for all measurable sets $A\subset \SE$.
Now from \cref{eq:fundID} and left-invariance of the Haar measure it directly follows that the effective operator $\widetilde{\Upsilon}(f) \coloneqq (\rP\circ \tilde{\Phi}\circ\widetilde{\cW}_\psi)(f) =\int_{-\pi}^\pi (\tilde{\Phi} \circ \widetilde{\mathcal{W}}_{\psi}f)(\cdot,\theta)\, {\rm d}\theta$, now with $\tilde{\Phi}:\mathcal{P}(\SE) \to \mathcal{P}(\SE)$, commutes with roto-translations if
\[
(L_{g})_\# \circ \tilde{\Phi} = 
\tilde{\Phi} \circ (L_{g})_\# \textrm{ for all }g \in \SE.
\]
Note that the operators $\mathcal{L}_g : L^{2}(\SE) \to L^{2}(\SE)$ and $(L_g)_\# : \cP(\SE) \rightarrow \cP(\SE)$ are very similar. Both ``roto-translate" their input with the group element $g$. The connection becomes more apparent when we consider a measure $\mu \in \cP(\SE)$ with a density $f_\mu$ with respect to the Haar measure on $\SE$. Then the density of $(L_g)_\# \mu$ is the function $\mathcal{L}_g (f_\mu)$.
\begin{remark}
Later in \cref{sec:numerical_results}, the lifting $\widetilde \cW_\psi$ will map to two probability measures on $\SE$: one will capture the positive components of $\cW_\psi$ (like in \eqref{eq:tildeW-v1}), and the other the negative components, as also done in \cite{galeotti2022cortically}. We will treat each component with the same optimal transport parameter settings, and add them afterwards. Working with both components will be crucial to keep contrast and sharp edges as in \cref{fig:intro}.
\end{remark}

\section{Symmetries and Theoretical Properties}
\label{sec:symm_ot}
Consider the setting of \cref{fig:intro} where we are interpolating between two images \(f_0, f_1 \in L^2(\bR^2)\) by lifting them to measures \(U_0, U_1 \in \cP(\SE)\), performing optimal transport there, and projecting back to \(L^2(\bR)\). Similar to the discussions in \cref{sec:lifting_equivariance_relation}, for the process 
\[
(f_0,f_1) \mapsto f_t:= 
P \circ 
\Phi_t(U_0,U_1), \textrm{ with }
U_i= (\widetilde\cW_{\psi}f_i),\; t \in [0,1],\;i=0,1.
\]
to be roto-translation equivariant we need the Wasserstein barycenter interpolation $\Phi_t$ on \(\SE\), recall \cref{equ:interpolation_operator}, to be left-equivariant: 
\begin{equation}
\label{requirement}
\Phi_t( (L_g)_\# U_0, (L_g)_\# U_1)=
(L_g)_\# \Phi_t(U_0,U_1)
\textrm{ for all }g \in G, t\in [0,1].
\end{equation}
This motivates looking at the general invariance/equivariance properties of the introduced optimal transport problems.

\subsection{Invariance and Equivariance of Optimal Transport}
\label{sec:inv_equi}
Next, we formulate the desired symmetries of each of the introduced optimal transport problems and prove that they are satisfied.
\begin{lemma}
\label{lem:equivariance}
Let $G$ be a Lie group with left-invariant metric $d$, then:
\begin{itemize}[itemindent=0.5cm]
\item The Wasserstein distance is invariant, i.e. for all measures $\mu, \nu \in \mathcal{P}(G)$ we have
\[
 W_p((L_g)_\#\mu,(L_g)_\#\nu) = W_p(\mu,\nu).
\]
\item Wasserstein barycenters are equivariant, for all $\mu_1, \dots, \mu_N \in \cP(G)$ and all $\Lambda_N \in \Sigma_N$ we have
$$
\bary((L_g)_\# \rY_N, \Lambda_N) = (L_g)_\# \bary(\rY_N, \Lambda_N),
$$
where $(L_g)_\#\rY_N = ((L_g)_\# \mu_1, \dots, (L_g)_\#\mu_N)$.
\item If the gradient flow $\cL^{GF}$ involves an invariant energy function $F: \cP(G) \rightarrow \R$, then $\cL^{GF}$ is invariant, and the whole evolution \cref{equ:GF_JKO}, that is $(\mu_n)_{n\in \N}$, behaves equivariant.
\end{itemize}
\end{lemma}

\begin{proof}
The proof relies on the fact that there is a bijection between $\Pi(\cX; \{\mu_i\}_{i=1}^K\}$ and $(L_g)_\# \Pi \coloneqq \Pi(\cX; \{(L_g)_\#\mu_i\}_{i=1}^K\})$. Consequently, if the loss function $\cL$ from the general OT problem \cref{eq:MK-multimarg} is invariant, then we have invariance in the value of minimizers, and equivariance of the arg-minimizers. Therefore, it suffices to check that $\cL^{W_p},$ $\cL^{\bary}$ and $\cL^{GF}$ are invariant to prove our claim. For $\cL^{W_p}$, let $(L_g)_\# \pi \in (L_g)_\# \Pi$. Then we easily see from left-invariance of the ground distance that 
\[
    \cL^{W_p}\big((L_g)_\# \pi \big) = \int_{G \times G} d^p(L_g(x), L_g(y)) \,\rd\pi(x, y) = \int_{G \times G} d^p(x, y) \,\rd\pi(x, y) = \cL^{W_p}\big(\pi \big),
\]
so that indeed the cost function for the Wasserstein distance is invariant. The other statements are easily derived following similar lines.
\end{proof}

We finish this section with a couple remarks:
\begin{itemize}
\item Many energy functionals $F$ for gradient flows satisfy the invariance assumption of \cref{lem:equivariance}. As an example, we may mention the class of ``inner energy functionals'' which take the form
\begin{equation}
F(\mu) = \begin{cases}
\int_G f \left( \frac{\rd \mu}{\rd \haar} \right) \rd \haar &\quad \text{ if } \mu \ll \haar \\
\infty &\quad \text{ else},
\end{cases}
\end{equation}
where $f: G \rightarrow \R$, $\ll$ denotes absolute continuity of measures and $\frac{\rd\mu}{\rd \haar}$ denotes the Radon-Nikodym derivative of $\mu$ w.r.t.~the Haar measure on $G$. Among the PDEs that they give rise to is for example the (non-linear) heat equation \cite{villani2008optimal}.
\item With minimal technical additions, the results of \cref{lem:equivariance} can also be carried out in the context of \textit{homogeneous spaces} equipped with a left-invariant distance. As a matter of fact, we just need left-invariance of the cost function $c$ to guarantee equivariance of the whole process in general. 
\item We remark that the conclusion of \cref{lem:equivariance} also applies to \emph{isotropic} metrics on $\R^2$, i.e. optimal transport on $(\R^2, |\cdot|^p)$ is also invariant/equivariant to roto-translations. However, the unique aspect of this property for $\SE$ is that unlike $\R^2$, optimal transport with anisotropic metrics in $\SE$ is \emph{also} left-invariant. \emph{Anisotropic} (or even sub-Riemannian) metrics in $\SE$ relate to association fields \cite{petitot2003neurogeometry,citti2006cortical,galeotti2022cortically,bellaard2022analysis} from neurogeometry, and have a line completion behavior which is useful in many practical applications. This feature (important for contour-propagation and contour-perception) is unique to $\SE$ and cannot be achieved with left-invariant metrics on $\R^2$.    
\end{itemize}

\subsection{Optimality of Right-Translated Measures}
\label{sec:OT_lie_groups}
Here we study whether right-translations can be optimal transport maps of the Wasserstein metric, recall \cref{eq:Wp}. Note that on a non-commutative group a fixed left-translation in general does not commute with left-translations, in contrast to right-translations, that do commute with all left-translations. The concrete question is:\vspace{0.1cm}
\begin{center}
\textit{For a given $\mu\in \cP_p(G)$, let $\pi^*$ be the the optimal transport plan of $W_p(\mu, (R_g)_\#\mu)$.\\ Is $\pi^*$ of the form $\pi^*=(\Id \times R_g)_\# \mu$? In other words, is the transport map $T=R_g$ optimal?}
\end{center}
\vspace{0.1cm}
We know there are a few cases in which this intuition holds, and we show that this result is not true in general for a wide class of Lie groups including $SE(2)$.

On the positive side, the result is true for $G=\R^n$ (see for example Remark 2.19 in \cite{peyre2020computational}). It is also true for certain right actions (right actions w.r.t. a specific sub-group) of $G=H(3)$ equipped with a specific sub-Riemannian structure as shown in \cite{ambrosio2004optimal}. However, we illustrate that this is not true in general by first building a counterexample in $SO(2)$.

\begin{lemma}
\label{prop:s1_example}
    Consider the Lie group $SO(2)$, identified with $\mathbb{R}/(2\pi\mathbb{Z}) \equiv (0, 2\pi]$ together with its left (or right) invariant metric $d$.
    Then we can find a measure $\mu$ and a group element $g$ such that $R_g$ (or $L_g$) is not the optimal transport map from $\mu$ to $(R_g)_\# \mu$ (or $(L_g)_\# \mu$) for the cost $d^p$ with $p \geq 1$.
\end{lemma}

\begin{proof}
 As the group $SO(2)$ is commutative left translations and right translations coincide. Let $0 < \epsilon \leq \pi/2$ be some positive constant and consider $\mu = \frac{1}{2} \delta_\epsilon + \frac{1}{2} \delta_{2\pi - \epsilon}$. Let $g = \pi$. We then have that $(R_\pi)_\# \mu = \frac{1}{2} \delta_{\pi - \epsilon} + \frac{1}{2} \delta_{\pi + \epsilon}$.

    It is clear that the cost associated to the transport map $R_\pi$ for the Monge formulation is $\pi^p$. However if we send the Dirac mass at $\epsilon$ to the one at $\pi - \epsilon$ and the one at $2\pi - \epsilon$ to the one at $\pi + \epsilon$, we get a cost strictly less than $\pi$, namely $(\pi - 2 \epsilon)^p$. Hence we conclude that left translations cannot be optimal. We can apply the same reasoning to the right translation case.
\end{proof}

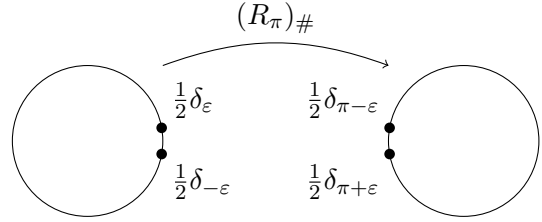
\begin{wrapfigure}{r}{0.5\textwidth}
\centering
\begin{tikzpicture}
  \coordinate (center1) at (0,0);
  \coordinate (center2) at (5,0);

  \draw (center1) circle [radius=1cm];

  \fill ($(center1)+(10:1cm)$) circle [radius=2pt] node[above right] {$\tfrac{1}{2}\delta_\varepsilon$};
  \fill ($(center1)+(-10:1cm)$) circle [radius=2pt] node[below right] {$\tfrac{1}{2}\delta_{-\varepsilon}$};

  \draw[->] (1,1) to [bend left=20] node[midway, above] {$(R_{\pi})_\#$} (4,1);
    
  \draw (center2) circle [radius=1cm];

  \fill ($(center2)+(170:1cm)$) circle [radius=2pt] node[above left] {$\tfrac{1}{2}\delta_{\pi-\varepsilon}$};
  \fill ($(center2)+(190:1cm)$) circle [radius=2pt] node[below left] {$\tfrac{1}{2}\delta_{\pi+\varepsilon}$};

\end{tikzpicture}
\caption{An example of measure on $SO(2)$ and its right translation for which the optimal transport plan between them is \underline{not} a right translation \cref{prop:s1_example}} 
\label{fig:prop:s1_example}
\end{wrapfigure}

Note that we could have built a similar example for functions with densities that are compactly supported on a sufficiently small interval. Also note that we could have taken a measure $\mu$ that was invariant under certain group actions, such as the uniform measure over $SO(2)$, but the example above has no such symmetries. We can leverage \cref{prop:s1_example} to construct counter-examples for a much larger class of Lie groups.

\begin{theorem}
\label{prop:compact_subgroup}
    Let $G$ be a Lie group that has a compact subgroup and $d$ a left invariant distance. Then we can find a measure $\mu$ and a group element $g$ such that $R_g$ (or $L_g$) is not optimal for the OT problem with cost $d^p$ with $p \geq 1$.
\end{theorem}

\begin{proof}
    By considering measures supported on the compact subgroup we assume w.l.o.g. that $G$ is compact.
    Every compact Lie group has a subgroup that is isomorphic to $SO(2)$. To see this, take some element in the Lie algebra and consider its generated one-parameter subgroup. If we take the closure of this subgroup we get an Abelian, compact, connected subgroup of $G$, which has to be isomorphic to a torus $\mathbb{T}^k$ for some $k > 0$. This torus then has a subgroup, say $H$, isomorphic to $SO(2)$. Using the isomorphism $H \rightarrow SO(2)$ we can construct a left-invariant metric on $SO(2)$ such that the isomorphism is an isometry. We can then construct an example as in \cref{prop:s1_example} to get a counter example on $G$.
\end{proof}
As a consequence of \cref{prop:compact_subgroup}, since $\SE$ has a compact subgroup, we have the following corollary.

\begin{corollary}
Left-actions on $\SE$ are not left-invariant. 
Right actions are left-invariant. 
However, both are not optimal maps in OT problems on $\SE$ with left-invariant cost $d^p$ with $p\geq 1$.
\end{corollary}
\begin{proof}
For all $g,h\in G$: 
$
L_g \circ R_h = 
R_h \circ L_g 
 \Rightarrow 
 (L_g)_{\#} \circ (R_h)_{\#}\mu = 
(R_h)_{\#} \circ (L_g)_{\#} \mu$ for all $\mu\in \cP_p(G)$. 
In a non-commutative Lie group like $G=\SE$  there exists $g,h \in G$ s.t. $L_g \circ L_h =L_{gh}\neq L_{hg}=L_h \circ L_g$ (where again we can take pushforwards). 
The rest follows by \cref{prop:compact_subgroup}.
\end{proof}

\section{Entropic Regularization and Scaling algorithms}
\label{sec:ER_scaling}
To solve OT problems of the form \cref{eq:MK-multimarg}, we follow the approach based on adding an entropic regularization to the loss function (see, e.g., \cite{cuturi2013sinkhorn, chizat2018scaling}). For $G=\SE$, when the OT cost $c=d^p$, this leads to numerical schemes that boil down to computing group convolutions with kernels that solely depend on the Riemannian distance. The strategy is very computationally intensive, though, because Riemannian distances in $\SE$ do not come in close form, and estimating them is an expensive operation. We leverage distance approximations that come in close form to make the scheme efficient (see \cite{bellaard2022analysis}). This section summarizes the whole approach and provides an error analysis of the approximations that we use.

\subsection{The Sinkhorn Algorithm}
\label{sec:algorithm}
The entropic-regularized version of problem \cref{eq:MK-multimarg} reads
\begin{equation}
\label{eq:entropy-loss}
\rho_\e = \inf_{\pi \in \cP(\cX)} \cL_\e(\pi) , \quad \text{with } \cL_\e(\pi)\coloneqq \cL(\pi)+ \sum_{i=1}^n F_i ((\mathrm{proj}_i)_\# \pi) + \epsilon \KL( \pi \Vert \pi^{\text{ref}} ),
\end{equation}
where $\e\geq 0$ is the regularization parameter.  The regularization term is the Kullback-Leibler (KL) divergence defined as
\begin{equation}
\KL ( \pi \Vert \pi^{\text{ref}} ) =
\begin{cases}
&\int_{\cX} \log \left( \frac{\rd \pi(x)}{\rd\pi^{\text{ref}}(x)} \right) \rd\pi(x),\quad 
\forall \pi \ll \pi^{\text{ref}}  
\text{ with }\pi,\,\pi^{\text{ref}}\in \cP(\cX). \\
&+\infty, \quad \text{else}.
\end{cases}
\end{equation}
and where the notation $\pi\ll \pi^{\text{ref}}$ means that $\pi$ is absolutely continuous with respect to $\pi^{\text{ref}}$. 
In our scheme, we work with
$$
\pi^{\text{ref}} = \otimes_{i=1}^n \alpha_i, \quad \alpha_i \in \cP(G)
$$
as the reference measure, where we choose the Haar measure $\alpha_i = \haar$ for all $i=1,\dots,n$. This is done to guarantee invariance and equivariance properties in the scheme as we prove in \cref{cor:equivariance-sinkhorn}. In \cref{eq:entropy-loss}, the functions $F_i : \cM(G) \rightarrow \R$ are assumed to be proper, lower-semi-continuous and convex. They encode marginal constraints or constraints of other nature. The regularization term makes the problem strongly convex and guarantees the uniqueness of a minimizer $\pi^*_\e$. One can show that as $\e\to 0$, $\pi^*_\e$ narrowly converges to the maximum entropy minimizer of the original problem (see \cite{leonard2010schrodinger, carlier2017convergence} for the proofs in the case of the Wasserstein distance, barycenters, and gradient flows).

For problems that have a functional $\cL$ of the form $\cL(\pi) = \int_\cX c \,\rd\pi$ one can reformulate \cref{eq:entropy-loss} using  simple algebraic manipulations as
\begin{equation}
\label{equ:generic_bregman}
    \inf_{\pi \in \cP(\cX)} \epsilon \KL\left(\pi \,\Vert \, e^{-c/\epsilon} \otimes_{i=1}^n \alpha_i \right) + \sum_{i=1}^n F_i ((\mathrm{proj}_i)_\# \pi).
\end{equation}
Recall for example the definition of $c$ for the Wasserstein distance: \cref{eq:LWp} and for the barycenters: \cref{equ:cost_bary}. We next recall the scaling algorithm to solve this problem (see \cite{chizat2018scaling} for its analysis). To simplify the presentation, we restrict the discussion to the case $n=2$. Given an initialization $b^{(0)} \in L^\infty(G)$ the algorithm repeats the following two operations until convergence:
\begin{equation}
\label{equ:scaling_iterations}
    a^{(l+1)} = \frac{\prox^{\KL}_{F_1/\epsilon}\left( \cK b^{(l)} \right)}{\cK b^{(l)}}, \;\;\; b^{(l+1)} = \frac{\prox^{\KL}_{F_2/\epsilon}\left( \cK^T a^{(l+1)} \right)}{\cK^T a^{(l+1)}},\quad l\in \bN,
\end{equation}
were $\prox^{\KL}_{F}$ denotes the \textit{proximal operator} of $F$ with respect to the KL divergence, and is defined as 
\begin{equation}
\label{equ:proximal_operator}
\prox^{\KL}_F(z) \coloneqq \argmin_{s \in \cP(G)} F(s) + \KL(s \;\Vert \; z).
\end{equation}
The application of the operator $\cK$ and its $L^2$-adjoint $\cK^T$ is defined as
\begin{equation}
\label{equ:convolution_scaling}
(\cK b)(x) \coloneqq \int_{G} e^{-c(x,y)/\epsilon} b(y) \;\rd \alpha_2(y), \;\;\; (\cK^T a)(y) \coloneqq \int_{G} e^{-c(x,y)/\epsilon} a(x) \;\rd \alpha_1(x).
\end{equation}
We are looking for a fixed point of these iterations, say $a^*, b^*$, after which we can recover the optimal coupling (as a density w.r.t.~$\alpha_1 \otimes \alpha_2$) by $\pi^*_\e(x,y) = a^*(x) e^{-c(x,y)} b^*(y)$ for all $(x, y)\in G\times G$. As an example, for the Wasserstein distance, we can enforce that $\pi \in \Pi(\cX; \mu_1, \mu_2)$ by choosing for $i=1,\,2$,
\begin{equation}
F_i(\nu) = \imath_{\{=\}}\big( \cdot | \mu_i \big)
\coloneqq
\begin{cases}
0,& \text{ if } \nu = \mu_i \in \cP(G) \\
+\infty,& \text{ else.}
\end{cases}
, \qquad 
\prox^{\KL}_{\imath_{\{=\}}\big( \cdot | \mu_i \big)}(z) = \mu_i\,.
\label{eq:Fi-marginal-constraint}
\end{equation}
Hence the (Sinkhorn) iterations read in this case (with point-wise division):
\begin{equation}
    a^{(l+1)} = \frac{\mu}{\cK b^{(l)}}, \;\;\; b^{(l+1)} = \frac{\nu}{\cK^T a^{(l+1)}},\quad \forall l\in \bN.
\end{equation}
See \cref{app:algorithm} for a complete algorithm for the barycenter problem. We conclude this section by showing that the entropic-regularized versions of the optimal transport problems still maintain all invariance/equivariance properties from \cref{lem:equivariance}. 

\begin{corollary}    
\label{cor:equivariance-sinkhorn}
The entropic-regularized scheme preserves the same equivariant and invariant properties of the unregularized OT problems.
\end{corollary}

\begin{proof}
Following the proof of \cref{lem:equivariance}, it suffices to show that $\pi \mapsto \cL_\e(\pi)$ is left-invariant. Since we showed this for the unregularized problems, it suffices to show left-invariance for the regularization term. We show this by
\begin{equation*}
    \begin{split}
        \KL\big((L_{g_0})_\# \pi \Vert \otimes_{i=1}^n \rd h ) &= \int_{\cX} \log \left( \frac{\rd (L_{g_0})_\# \pi}{\otimes_{i=1}^n \rd h} (g) \right) \otimes_{i=1}^n \dhaar{g} \\
        &= \int_{\cX} \log \left( \frac{\rd \pi}{\otimes_{i=1}^n \rd h}(L_{g_0^{-1}} g) \right) \otimes_{i=1}^n \dhaar{g}\\
        &= \int_{\cX} \log \left( \frac{\rd \pi}{\otimes_{i=1}^n \rd h}(g) \right) \otimes_{i=1}^n \dhaar{g} = \KL\big( \pi \Vert \otimes_{i=1}^n \haar ).
    \end{split}
\end{equation*}
Finally, the $F_i$ are used to encode marginal constraints (except for the gradient flow), taking the form \cref{eq:Fi-marginal-constraint}, and we can treat them the same way as the marginal constraints in \cref{lem:equivariance}. This proves that the regularized problem inherits the same symmetries as the unregularized problem. 
\end{proof}

\subsection{Group Convolutions}
\label{subsec:gconv}
When the cost function $c=d^p$ for a left-invariant distance $d$, the linear operation \cref{equ:convolution_scaling} becomes a \textit{group convolution} that one can implement more efficiently than in the general case. The group convolution of a function $f : G \rightarrow \R$ with a kernel $K: G \rightarrow \R$ is defined by
$$
(K * f)(g) \coloneqq \int_G K(h^{-1} g) f(h) \dhaar{h}, \quad \forall g\in G,
$$
Taking the reference measure $\pi^{\text{ref}} = \otimes_{i=1}^n \haar$ in the $\KL$ term, and defining
\begin{equation*}
\label{equ: kernel_se2}
K(g) \coloneqq e^{-d^p(e,g) / \epsilon}, \quad\forall  g\in G,
\end{equation*}
it follows by left-invariance of the distance map \cref{eq:left-inv-distance} that \cref{equ:convolution_scaling} can be written as convolutions
\begin{equation*}
(\cK b)(g) =\int_{G} e^{-d^p(g,h)/\epsilon} b(h) \dhaar{h}=
\int_{G} e^{-d^p(e,h^{-1}g)/\epsilon} b(h) \dhaar{h}=(K * b)(g).
\end{equation*}
Note that the operator $\cK$ and its adjoint $\cK^T$ are equal because the symmetry of the Riemannian distance $d$ carries over to the cost $c=d^p$.

The convolutional structure opens the door to very efficient implementations of \cref{equ:convolution_scaling}. $\SE$ group convolutions are extensively used in geometric deep learning for image processing (e.g. \cite{cohen2016group, weiler2018learning, weiler2019general,bekkers2018roto})
and robotics (e.g. \cite{chirikjian_engineering_2000}). 
Due to their widespread use in diverse applications, there have been many efficient implementations proposed (see \cite{chirikjian_engineering_2000, FrankenPhDThesis, bekkersbspline2019}).
Additionally, if the kernels have specific separability symmetries then further speed-ups can be expected \cite{sifre2013rotation,finzi}. Unfortunately, in \(\SE\) such separability constraints are not generally applicable (if $\mc{3}>0$, $\mc{1} \neq \mc{2}$). For a detailed overview of many steerable and non-steerable implementations with order complexities please refer \cite[Tables 3.1 and 3.2]{FrankenPhDThesis}. In all of our experiments, we always work with non-steerable direct implementations.

\subsection{Distance Approximations on $\SE$}
\label{subsec:dist_approx}
Geodesic distances on $\SE$ cannot be calculated in closed form, especially for anisotropic metrics which are of primary interest in this paper. In general, their estimation in practice is achieved using solutions to Eikonal PDE's (see e.g.~\cite{MirebeauJMIV,bekkers2015pde,sethian1999level}). However, these approaches can become computationally expensive for large discretizations and high anisotropies. Instead, we now introduce analytic approximations that can be conveniently calculated in closed form and derive their numerical properties. 

An analytic \(\SE\) distance approximation was introduced in \cite{smets2023pde}, building upon standard approximation theory in Lie groups \cite{rothschild1976hypoelliptic, terElst3}.
We sketch the general idea below. We start by noting that since $d$ is left-invariant, it suffices to approximate $d(e, g)$ for all $g\in \SE$ since in this case we can then evaluate $d(h, g)=d(e, h^{-1}g)$ for all $h, g\in G$.

For a Lie group with a surjective Lie group exponential \(\exp : T_eG \to G\) we can, after choosing an appropriate subset $S \subset$ \(T_eG\), make it injective.
This way the Lie group logarithm \(\log : U \to T_eG\), the inverse of the exponential, is well-defined on $U = \exp(S)$, and we can use it to create an approximation of \(d(e,g)\) in the following way.

The \textit{logarithmic distance approximation} \(\rho_c : U \to \R_{\geq 0}\) is defined as \cite{smets2023pde}:
\begin{equation}
    \rho_c(g) = \| \log g \|_e
\end{equation}
where \(\| \cdot \|_e : T_eG \to \R_{\geq 0}\) is the norm induced by \(\cG_e\).
The calculation of \(\rho_c\) is relatively easy and often expressible in closed form for many important Lie groups (for example matrix Lie groups), which means this approximation is generally applicable.
The logarithmic approximation has the important global property that \(d(e,g) \leq \rho_c(g)\) and one can show that it locally matches the distance in a precise sense explained further in the previously mentioned literature.

In \cite{bellaard2022analysis} the logarithmic approximation \(\rho_c\) on \(\SE\) is further refined to the \textit{half-angle approximation} \(\rho_b\) defined on $\SE$ as
\begin{equation}
\label{eq:rhob}
    \rho_b(g) \coloneqq \sqrt{ (\mc{1} b^1(g))^2 + (\mc{2} b^2(g))^2 + (\mc{3} b^3(g))^2}.
\end{equation}
where \(b^i : \SE \to \R\) are the \textit{half-angle coordinates}, which for a given element $g = (x,y,\theta)$ can be computed as\footnote{Where we use the small-angle convention $\theta \in [-\pi, \pi)$.}:
\begin{equation} 
\label{equ:logcomp}
    \begin{array}{ll}
        \Big( b^1(g),b^{2}(g),b^{3}(g) \Big) &=
        \big(\,
        x\cos (\theta/2) + y \sin (\theta/2), \, 
        -x \sin(\theta/2) + y \cos (\theta/2), \,
        \theta \,
        \big).
    \end{array}
\end{equation}

The approximative distance \(\rho_b\) can be extended naturally to a metric-like structure on \(\SE\) by defining \(\tilde \rho_b(g_1, g_2 ) := \rho_b(g_1^{-1} g_2)\).
It is worth mentioning that this $\tilde \rho_b$ is not always strictly a metric on $\SE$, as it is not guaranteed to satisfy the triangle inequality.
However, we will refer to it as a distance purely because it serves as an approximation to it. 

The logarithmic and half-angle approximation both have the same global symmetries as the exact distance \(d(e,g)\) as shown in \cite[Lemma 3]{bellaard2022analysis}. 
However, in contrast to \(\rho_c\), the approximation \(\rho_b\) also has the desirable properties described in the following lemma.
\begin{lemma}[Corollary 2, Lemma 7 from \cite{bellaard2022analysis}]
\label{thm:distance_approx}
    The Riemannian distance $d = d_\cG$ defined in \cref{eq:metric_SE2}, and the approximate distance $\rho_b$ defined in \cref{eq:rhob} are globally equivalent in the sense that
\begin{equation}
\label{eq:equivalence-rhob}
        \frac{1}{\zeta} d(e,g) \leq \rho_b(g) \leq \zeta d(e,g),\quad \forall g\in \SE,
\end{equation}
    where $\zeta$ is the spatial anisotropy defined in \cref{equ:spatial_anisotropy}. Furthermore, around any compact neighborhood $\Omega \subset \SE$ of the identity element $e$, we have that
    \[
        d(e, g)^2 \leq \rho_b(g)^2 \leq \Big(1+\epsilon(g)\Big) d(e, g)^2, \quad \forall g= (x,y, \theta)\in \Omega,
    \]
    where
    \[\epsilon(g) := \frac{\zeta^2 - 1}{2 \mc{3}^2} \zeta^{4} \rho_b^2(g) + C_\Omega |\theta|^3\]
    and $C_\Omega \geq 0$ depends on $\Omega$.
\end{lemma}

\subsection{Error analysis}
\label{subsec:error_an}
Our resulting numerical algorithm involves two main approximations with respect to the original OT formulation: we introduce regularization, and we approximate $d$ with $\rho_b$. In this section, we theoretically analyze how much the final outputs deviate from the original OT solutions. We will focus on analyzing the Wasserstein distance problem. For this, recall that the original loss function to minimize is $\pi \mapsto\cL^{W_p}(\pi)$ as defined in \cref{eq:LWp}, and the value of the minimizer is $W_p(\mu, \nu)^p = \cL^{W_p}(\pi^*)$, where $\pi^*$ is the optimal transport plan. We denote the entropic-regularized version by $\pi \mapsto \cL_{\e}^{W_p}(\pi)$, and the value of the minimizer is denoted by $W_{p, \e}(\mu, \nu)$. For our discussion, it will be useful to add the dependency on the underlying distance $d$  in the notation, so we will write these quantities as $\cL^{W_p, d}$ and $W_p^{(d)}(\mu, \nu)$ for the non-regularized formulation, and $\cL_\e^{W_p, d}$ and $W_{p, \e}^{(d)}(\mu, \nu)$ for the regularized versions respectively. We furthermore use the notation $\pi^*_{d, \epsilon}$, $\pi^*_{\rho_b, \epsilon}$ for the regularized optimizers, and $\pi^*_{d, \epsilon, n}$, $\pi^*_{\rho_b, \epsilon, n}$ for their Sinkhorn approximations after $n$ steps.

We prove in the next lemma that $W_p^{(d)}$ and $W_p^{(\rho_b)}$ are equivalent in the sense that we can upper and lower bound one quantity with the other. The same result applies to the regularized versions.
\begin{theorem}
\label{lem:w-bound}
For all $\mu, \nu \in \cP(\SE)$ it holds that
\begin{equation}
\frac 1 \zeta W_{p}^{(d)}(\mu, \nu) \leq W_{p}^{(\rho_b)}(\mu, \nu)  \leq \zeta W_{p}^{(d)}(\mu, \nu) 
\end{equation}
and
\begin{equation}
\frac 1 \zeta W_{p, \zeta^p \e}^{(d)}(\mu, \nu) \leq W_{p, \e}^{(\rho_b)}(\mu, \nu)  \leq \zeta W_{p, \zeta^{-p}\e}^{(d)}(\mu, \nu).
\end{equation}
\end{theorem}

\begin{proof}
Let $\pi \in \Pi(\mu, \nu)$. Using \cref{eq:equivalence-rhob}, we have that 
\begin{equation*}
\begin{split}
     W_{p, \e}^{(\rho_b)}(\mu, \nu)^p &\leq \int \rho_b(h^{-1}g)^p \rd \pi(h,g) + \epsilon \KL (\pi \Vert \haar \otimes \haar) \\ 
     &\leq \zeta^p \left(\int d(h,g)^p \rd \pi(h,g) + \frac{\epsilon}{\zeta^p} \KL (\pi \Vert \haar \otimes \haar) \right).
\end{split}
\end{equation*}
Since this holds for all $\pi \in \Pi(\mu,\nu)$, we can take the infimum over $\pi$ and the $p$-th root to get
\begin{equation}
    W_{p, \e}^{(\rho_b)}(\mu, \nu)  \leq \zeta W_{p, \zeta^{-p}\e}^{(d)}(\mu, \nu).
\end{equation}
The other side of the bound can be derived in an analogous way. To get the result on the unregularized problem, one can set $\epsilon = 0$.
\end{proof}

We next only consider $p=2$ and leverage recent results from  \cite{eckstein2022quantitative} to estimate how much the optimal coupling $\pi^*_{d, \epsilon}$ with exact distance $d$ differs from the optimal coupling $\pi^*_{\rho_{b}, \epsilon}$ with the distance approximation. We also compare it with $\pi^*_{\rho_{b}, \epsilon, n}$, the coupling that one can compute in practice after $n$ Sinkhorn iterations. For this, we need to introduce the following metric on $\SE \times \SE$,
\[
\tilde{d}((g_1, h_1), (g_2, h_2)) \coloneqq d(g_1, g_2) + d(h_1, h_2),
\quad \forall (g_1,h_1), (g_2,h_2)\in \SE\times\SE .
\]

\begin{theorem}
\label{lem:error_bound_optimizers}
Let $\mu, \nu \in \cP(\SE)$ be bounded and compactly supported measures, and let $\pi^*_{d, \epsilon}$ and $\pi^*_{\rho_{b}, \epsilon}$ be the associated optimal regularized transport plans for the different costs. Then
\begin{equation*}
W_1^{\tilde d}(\pi^*_{d, \e}, \pi^*_{\rho_b, \e, n}) \leq
\frac{C_1}{\epsilon \sqrt{2}} \; \sup_{g = (x,y,\theta) \in S} \left(\frac{\zeta^2 - 1}{2 \mc{3}} \zeta^4 \rho_b(g)^2 + C_2 |\theta|^3\right) d(e,g)^2
+ \frac{C_3}{\epsilon} n^{-4},
\end{equation*}
where $C_1 > 0$ is a constant depending on $\mu, \nu$, $C_2,\, C_3\geq 0$ are constants, and
$$
S \coloneqq \{ g \in \SE \,|\, g = h^{-1} k, h \in \spt(\mu), k \in \spt(\nu)\}.
$$
\end{theorem}
\label{app:proof}
\begin{proof}
We first apply the triangle inequality to the term we wish to bound:
\begin{equation}
W_1^{\tilde d}(\pi^*_{d, \e}, \pi^*_{\rho_b, \e, n}) \leq W_1^{\tilde d}(\pi^*_{d, \e}, \pi^*_{\rho_b, \e}) + W_1^{\tilde d}(\pi^*_{\rho_b, \e}, \pi^*_{\rho_b, \e, n}).
\label{eq:ineq-pert-analysis}
\end{equation}
We then bound the two terms separately. For the first term, we wish to apply Proposition 3.12 from \cite{eckstein2022quantitative}. Since we assume $\mu, \nu$ are bounded and have compact support, we can apply the proposition with $p = \infty$ and $q = 1$. We then have to check that we satisfy condition $I_1$ from \cite{eckstein2022quantitative}. This is true thanks to Lemma 3.10 from \cite{eckstein2022quantitative}, again because we have compact and bounded marginals. Then proposition 3.12 from \cite{eckstein2022quantitative} gives us
\begin{equation}
\label{eq:perturbation}
    W^{\tilde{d}}_1(\pi^*_{d, \epsilon},\pi^*_{\rho_{b}, \epsilon}) \leq \frac{C_1}{\epsilon \sqrt{2}} \norm{d^2 - \rho_b^2}_{L^\infty(\mu \otimes \nu)}.
\end{equation}
We next apply \cref{thm:distance_approx} to derive:
\begin{align}
\norm{d^2 - \tilde \rho_b^2}_{L^\infty(\spt(\mu \otimes \nu))}
&\leq \sup_{(g,h) \in\spt(\mu \otimes \nu)} |d(h,g)^2 - \rho_b(h^{-1}g)^2| \nonumber\\
&= \sup_{g \in S} |d(e,g)^2 - \rho_b(g)^2| \nonumber\\
&\leq \sup_{g = (x,y,\theta) \in S} \left(\frac{\zeta^2 - 1}{2 \mc{3}} \zeta^4 \rho_b(g)^2 + C |\theta|^3\right) d(e,g)^2. \label{eq:dist-pert}
\end{align}
Plugging \cref{eq:dist-pert} into \cref{eq:perturbation} yields the bound on the first term of \cref{eq:ineq-pert-analysis}. For the second term, Theorem 3.15 from \cite{eckstein2022quantitative} guarantees the existence of a constant $C_3\geq 0$ (which is independent of $\e$) such that
$$W_1^{\tilde d}(\pi^*_{\rho_b, \e}, \pi^*_{\rho_b, \e, n})\leq \frac{C_3}{\epsilon} n^{-4}.$$
Combining the two inequalities concludes the proof.
\end{proof}

It follows from the bounds from \cref{lem:error_bound_optimizers} that $\pi^*_{d, \e}$ and $\pi^*_{\rho_b, \e, n}$ are close only when the supports of $\mu$ and $\nu$ are concentrated and close to each other, which is something we can control using the discretization of the measures. The proposition also reveals that the Sinkhorn iterations decay at a rate of $n^{-4}$, which suggests a rather fast convergence of the algorithm. One might even be able to prove faster convergence, see \cite{eckstein2024hilberts, deligiannidis2021quantitative}, but this is not necessary for us, as the error we make in the approximate distance is already larger than the iteration error.
\section{Numerical Experiments}
\label{sec:numerical_results}
\begin{figure}[t]
    \centering   
    \includegraphics[width=1\linewidth, page=1]{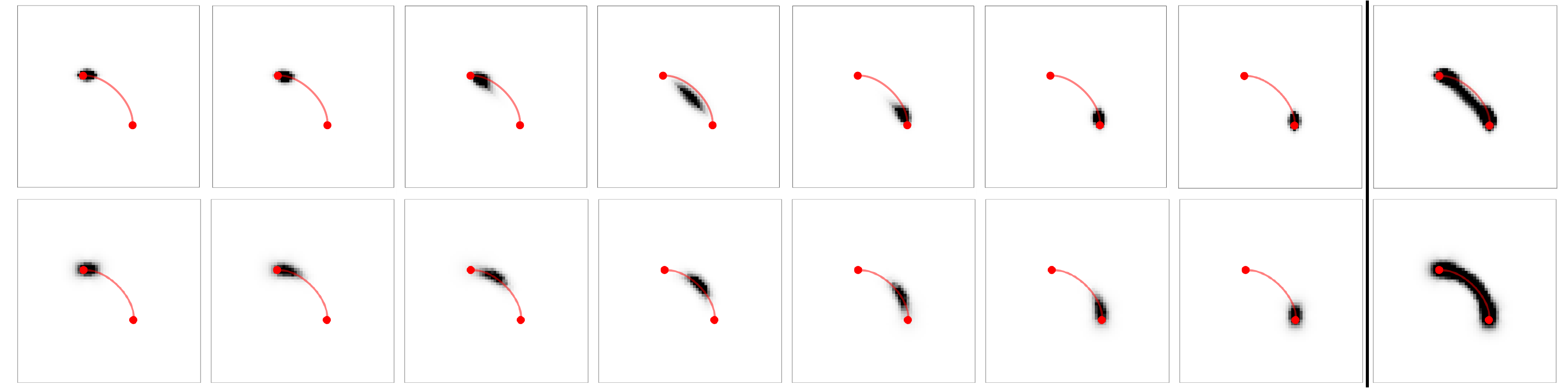}
    \caption{\textbf{Tracing U-Curve Geodesics in $\SE$:} We interpolate between two Dirac masses using a anisotropic metric $(w_1,w_2,w_3) = (1, 5, 2.5)$ in $\SE$. The overlaid red curve is the exact sub-Riemannian geodesic between the two endpoints computed via analytic solutions from \cite{duits2014association}. {\bf 1\textsuperscript{st} Row:} Interpolation with Heat Diffusion in $\SE$. {\bf 2\textsuperscript{nd} Row:} Interpolation with the proposed half-angle distance approximation $\rho_b$ from equation~\!\cref{eq:rhob}. We can observe that the transport from our distance approximation is more accurate to the analytic solution in comparison to the analogous solution with Heat Diffusion. {\bf Columns 1-7:} Spatially projected $\SE$ interpolations on $\R^2$. {\bf Column 8:} A cumulation of columns 1-7 showing the path that was traced. See \cref{fig:s_curve} for an additional example.} 
    \label{fig:geodesic_path}
\end{figure}
We now report the results from the numerical implementation of our method for different applications. We apply the Sinkhorn algorithm with a $\SE$ group convolution together with a Gibbs kernel made from the half-angle distance approximation $\rho_b$. We use the $\SE$ group convolution implementation from the publicly available LieTorch package in \cite{smets2023pde}. For convenience we describe the $\SE$ Wasserstein barycenter algorithm, which will be used for the image examples, more precisely in the supplementary materials \cref{app:algorithm}. 

\subsection{Barycentric Interpolation with Synthetic Examples on \(\SE\)}
We first validate our method by tracing the path obtained from a 2-Wasserstein interpolation of point masses in \cref{fig:geodesic_path} and \cref{fig:s_curve}. We discretize $\SE$ into a volume of $64 \times 64 \times 16$ by discretizing the spatial X-Y plane $[0,1] \times [0,1]$ into a $64 \times 64$ grid and $\theta \in [0, 2\pi)$ into 16 orientations. We are particularly interested in tracing geodesics obtained under high-anisotropies or near sub-Riemannian conditions on $\SE$. In these experiments, we validate a well-known result in OT: that the barycenter between two measures, recall \cref{equ:interpolation_operator}, coincides with the McCann interpolation or displacement interpolation, which transports the mass of the measure along a geodesic, see for example \cite[Section 7]{villani2008optimal}.  
We compare the path of the interpolations with the true sub-Riemmanian geodesic. The overlaid red curves in \cref{fig:geodesic_path} and \cref{fig:s_curve} are computed with exact analytic formulas in \cite{moiseev2010maxwell, duits2014association}. With increasing spatial anisotropies, Riemannian geodesics rapidly converge to the sub-Riemannian geodesics \cite{duits2018optimal} where exact solutions are available. We compare the output of entropy-regularized Wasserstein interpolation with:
\begin{enumerate}
\item Solving Anisotropic Heat Diffusion on $\SE$. This approach is a dominant paradigm for computing entropic OT on geometric domains (see \cite{solomon2015convolutional}).
\item Our distance approximation $\rho_b$ with logarithmic coordinates from equation \cref{eq:rhob}.
\end{enumerate}

\subsection*{Tracing U-Curve Geodesics in $\SE$} We place two point masses at $g_1 = (0.375,0.625, \\ \theta = 0^{\circ})$ and $g_2 = (0.625,0.375, \theta = 90^{\circ})$. As we see in \cref{fig:geodesic_path} the entropic-regularized 2-Wasserstein interpolation with $\rho_b$ very closely follows the exact geodesic. We observe that in this case, it is numerically more accurate than the solution with anisotropic heat diffusion. This could be because optimal transport via the classical heat diffusion builds upon Varadan's theorem \cite{varadhan1967behavior} and hence requires very small time steps for high accuracy. 
\begin{figure}[t]
    \centering   
    \includegraphics[width=1\linewidth]{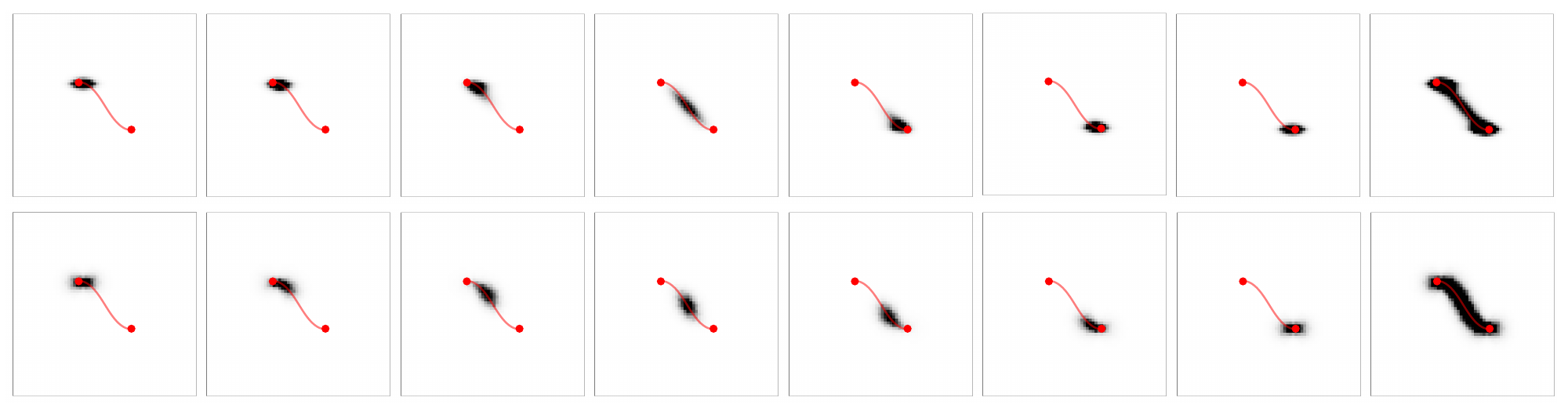}
   \caption{\textbf{Tracing S-Curve Geodesics in $\SE$:} We interpolate between two Dirac masses $g_1 = (0.375,0.625, \theta = 0^{\circ})$ and $g_2 = (0.625,0.375, \theta = 0^{\circ})$ using anisotropic metric $(w_1,w_2,w_3) = (1, 7, 2.5)$ in $\SE$. The overlaid red curve is the exact sub-Riemannian geodesic between the two endpoints computed via analytic solutions from \cite{duits2014association} with challenging extreme curvatures at the boundaries. {\bf 1\textsuperscript{st} Row:} Interpolation with Heat Diffusion in $\SE$. {\bf 2\textsuperscript{nd} Row:} Interpolation with the version of the distance approximation $\rho_{b,com}$ \cite[Eq.26]{bellaard2022analysis}. The transport from the distance approximation is more accurate to the analytic solution in comparison to the analogous solution with Heat Diffusion. {\bf Columns 1-7:} Spatially projected $\SE$ interpolations on $\R^2$. {\bf Column 8:} A cumulation of columns 1-7 showing the path that was traced. } 
   \label{fig:s_curve}
\end{figure}
\subsection*{Tracing S-Curve Geodesics in $\SE$} Sub-Riemannian geodesics come either as U or S shapes \cite{moiseev2010maxwell,duits2014association} and we also demonstrate a challenging $S$-case with extreme curvatures at the boundary points in the same setup as \cref{fig:geodesic_path}. In \cite{bellaard2022analysis} a series of distance approximations are studied, all of which were motivated by the logarithmic approximation we explain in \cref{subsec:dist_approx}. $\rho_b$, recall \cref{eq:rhob}, is a very good distance approximation for spatial anisotropies (\ref{equ:spatial_anisotropy}) in the range $\zeta<5$. 
This was sufficient and yielded stable results in all our experiments. 
However, a further adjustment in the approximation should be made (with an automatic switch to the sub-Riemannian setting) if the spatial anisotropy is more extreme \cite[Eq.26]{bellaard2022analysis}.  This is particularly relevant in the tracing of S-shaped sub-Riemannian geodesic in $\SE$ with high curvature at the endpoints. As $\zeta \to \infty$ the Riemannian geodesics converge to more challenging sub-Riemannian geodesics, cf.~\cite[Thm.2]{duits2018optimal}. In our experiments, $\zeta=7$ was large enough to verify the concentration of the optimal transport around an exact, high curvatures, sub-Riemannian S-shaped geodesics in \cref{fig:s_curve}.

\begin{figure}
    \centering   
    \subfloat{\includegraphics[width=1\linewidth]{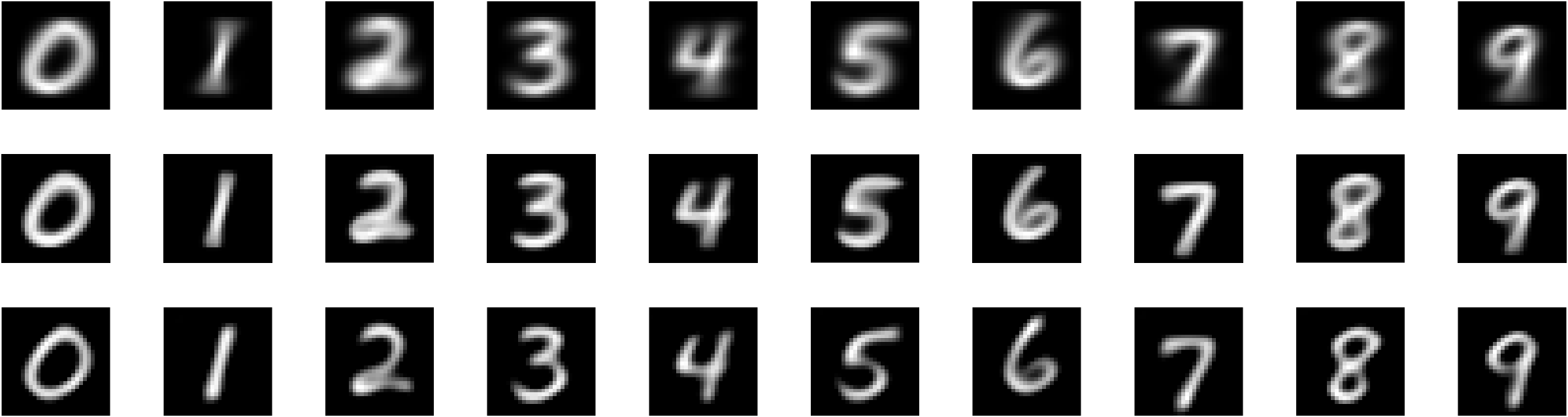}} \hspace{1cm}
    \caption{Average barycenters of all ($ \sim 6000$ images per class) images from the MNIST dataset. {\bf 1\textsuperscript{st} Row}: $L^2$ barycenter of all images per class - \LR. {\bf 2\textsuperscript{nd} Row}: 2-Wasserstein barycenter in $\mathbb{R}^2$ - \WR, {\bf 3\textsuperscript{rd} Row}: 2-Wasserstein barycenter in $\SE$ - \SER. \SER shows sharper interpolations in comparison with its counterparts on $\mathbb{R}^2$.} 
    \label{fig:MNIST_multiple}
\end{figure}

\subsection{Barycentric Interpolation of Images with Lifting to $\SE$}
\label{sec:interp-img}

\subsubsection*{Setting}
We compute barycentric averages of images having shapes represented as 2D contours and line structures. We report qualitative results from 2 datasets that were particularly valuable in this setting: MNIST \cite{deng2012mnist} and the QuickDraw\footnote{\url{https://quickdraw.withgoogle.com/data}}. MNIST is the well-known database of handwritten digits 0-9, with $\sim 6000$ images per class. Quickdraw is also a diverse database of doodles made by many people over the internet for different object classes.

We compare the following barycenters:
\begin{itemize}
\item[\LR:] Barycenters in the linear space $(L^2(\bR^2), \Vert\cdot\Vert_2)$,
\item[\WR:] Barycenters in $W_2(\R^2)$ with entropic regularization and Euclidean distance $\Vert\cdot\Vert_2$ in~$\bR^2$,
\item[\SER:] Barycenters in $W_2(\SE)$, with entropic regularization and $\rho_b$ as a distance approximation in $\SE$. This approach requires lifting the images from $L^2(\bR^2)$ to $\cP_2(\SE)$, and projecting back to $L^2(\bR^2)$ before and after the barycenter computation. These operations correspond to the operator $\widetilde \cW_\psi$ in \cref{sec:lifting_equivariance_relation}. We next explain how we have built them in practice in the next subsection.
\end{itemize}  

\subsubsection*{Lifting and Projecting Images to and from Measures: Practical Aspects}
To build a lifting $\widetilde \cW_\psi: L^2(\bR^2)\to \cP_2(\SE)$, we take the orientation score transform $\cW_\psi: L^2(\bR^2)\to L^2(\SE)$ introduced in \cref{eq:lift} as a starting point.
For a given image $f\in L^2(\bR^2)$, the lifted image $\cW_\psi(f)$ is not a probability measure because it may take on negative values\footnote{The cake-wavelets $\psi$ have negative side-lobes that fan out opposite to the orientation they measure \cite[Fig.6.4]{duits2005perceptual}}. 
They are problematic for our purposes but they play a crucial role in accurate image recovery. 
The first idea to fix the issue would consist of working with the absolute value. 
However, this option did not work in our experiments, and its failure is connected to the loss of the sign in the value of the orientation scores.
\begin{figure}
    \centering   
    \subfloat{\includegraphics[width=1\linewidth]{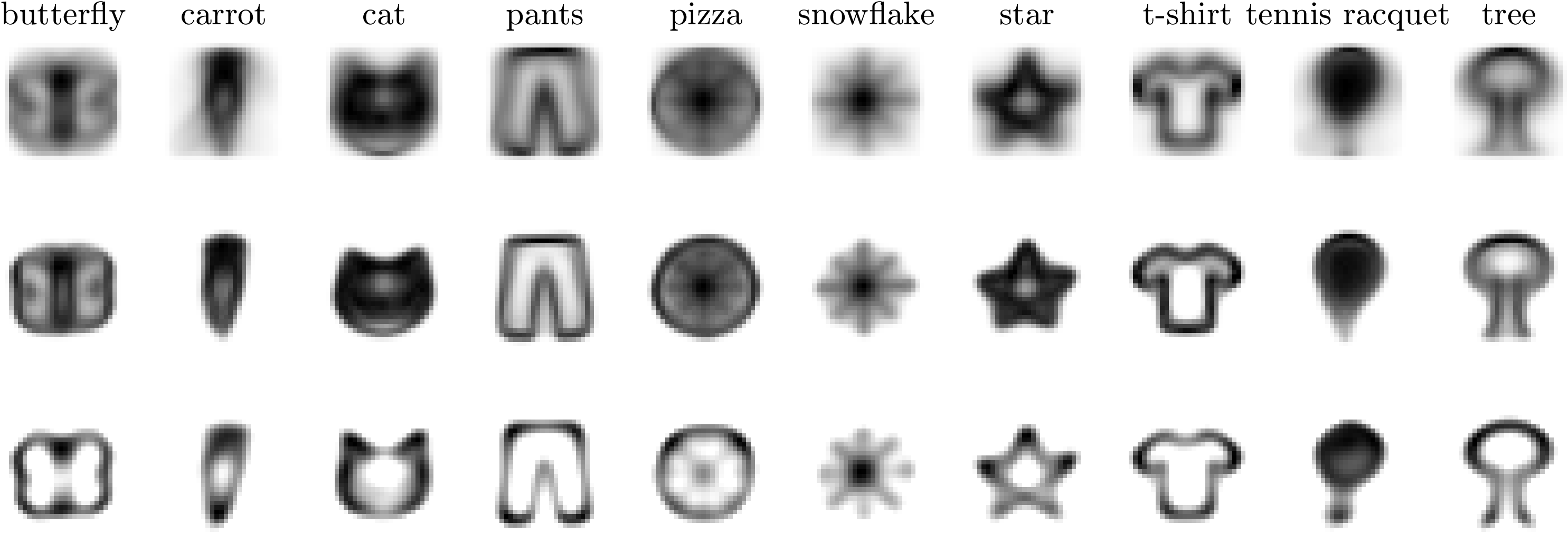}}
    \caption{Average barycenters of ($ \sim 6000$ images per class) images from the QuickDraw dataset. {\bf 1\textsuperscript{st} Row}: $L^2$ barycenter of all images per class - \LR. {\bf 2\textsuperscript{nd} Row}: 2-Wasserstein barycenter in $\mathbb{R}^2$ - \WR, {\bf 3\textsuperscript{rd} Row}: 2-Wasserstein barycenter in $\SE$ - \SER. \SER shows sharper interpolations in comparison with its counterparts on $\mathbb{R}^2$.} 
    \label{fig:doodle_multiple}
\end{figure}
Motivated by a similar construction from \cite{galeotti2022cortically}, we address this issue in practice by splitting the orientation scores into their respective positive and negative components and applying OT tools on them independently. For example, for computing \SER  between two images (recall \cref{fig:intro}) we do the following:

\begin{enumerate}
    \item Lift images $f_0, f_1\in L^2(\bR^2)$ to their scores $U_0 = \cW_\psi(f_0)$ and $U_1=\cW_\psi(f_1)$ (see \cref{eq:lift}). 
    \item Compute the positive and negative components of the scores: \newline $\{ U_k^+, U_k^-\}:= \{ \max\{U_k,0\}, -\min\{U_k,0\} \geq 0\}$ for $k \in \{0,1\}$ \\
    We normalize the output to sum up to 1 and therefore, we obtain measures $U_k^\pm \in \cP_2(\SE)$ for $k \in \{0,1\}$
    \item Compute the $\SE$ Wasserstein barycenters $\Phi_t(U_0^+, U_1^+)$ and $\Phi_t(U_0^-, U_1^-)$, using the interpolation operator \cref{equ:interpolation_operator}, and $\rho_b$ for the distance approximation. \Cref{cor:equivariance-sinkhorn} guarantees that this step is left-equivariant: $\Phi_t(\mathcal{L}_gU_0^{\pm},\mathcal{L}_gU_1^{\pm})= \mathcal{L}_{g}\Phi_t(U_0^{\pm},U_1^{\pm})$.   
    \item Combine projections to get the image interpolant $b_t\in L^2(\bR^2)$:
    \begin{equation} 
    \label{eq:final}
        b_t = \widetilde \rP(U^{+}_t, U^{-}_t) \coloneqq \rP(U^{+}_t) - \rP(U^{-}_t), \quad \forall t\in [0,1], 
    \end{equation}
    where $\rP$ is the projection operator defined in \cref{eq:proj}. Our projection $\widetilde \rP$ ensures the interpolation property
    \[
    b_{t=0}=f_0 \textrm{ and } b_{t=1}=f_1.
    \]
\end{enumerate}
Since all steps are left-invariant, the whole procedure is left-invariant. Also, we emphasize that it is vital to use the same OT settings for the positive and negative components of the scores to ensure the respective interpolations are synced. In all our experiments we work with images that are also represented as densities $\cP_2(\R^2)$. We threshold the projected output $b_t$ from Step 4 by keeping only the dominant positive response that is eventually normalized into $\cP_2(\R^2)$ for \SER as well as \WR.      
\begin{figure}[t]
    \centering   
    \includegraphics[width=1\linewidth, page=2]{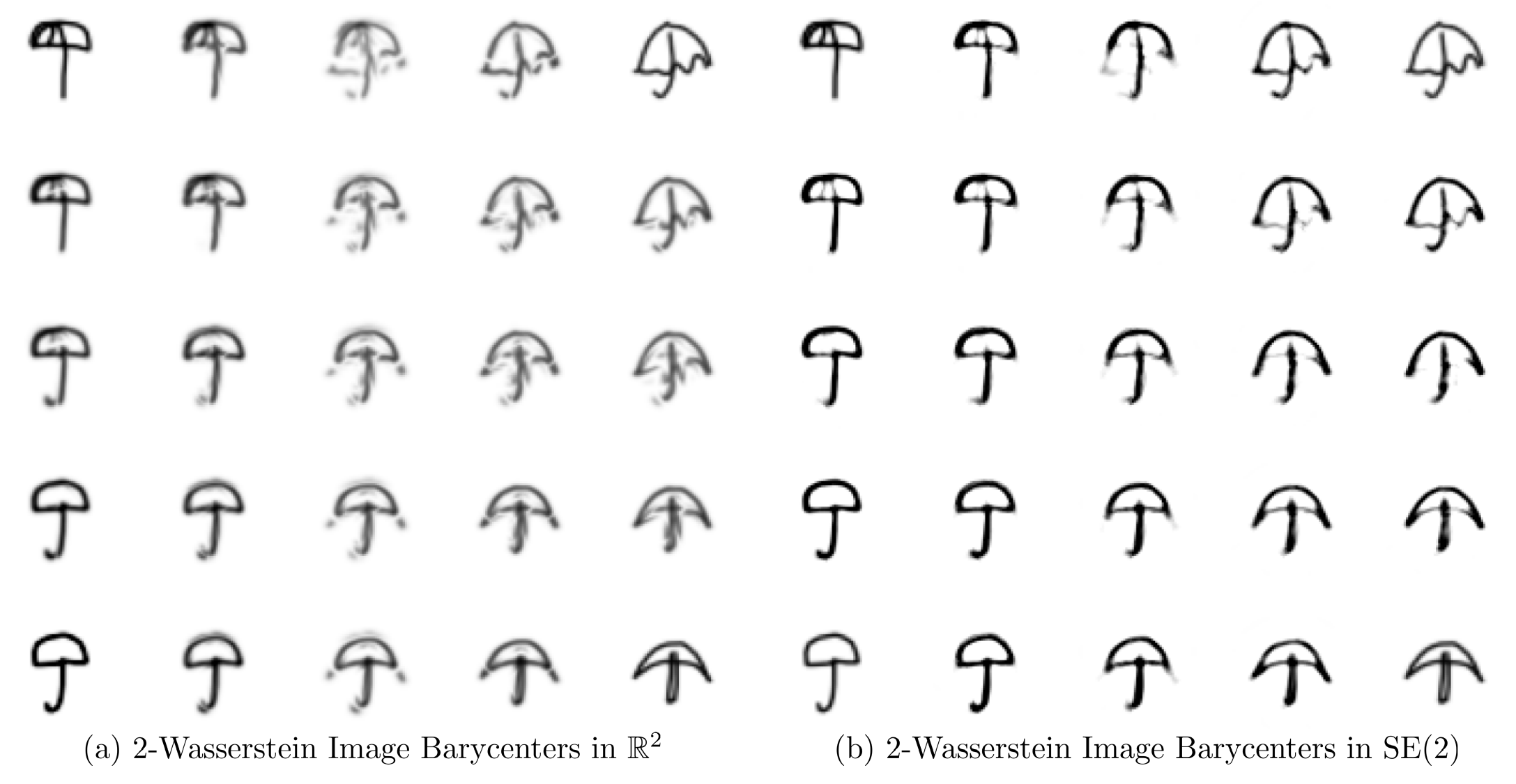}
    \caption{\textbf{Comparing image barycenters in $\R^2$ and $\SE$:} We compute the barycenter between the 4 images at the corners, with linear weights in between. We select the regularization parameter such that the effective line thickness for both cases is approximately the same.
    \SER promotes transport along line structures and hence, under this setting, we observe more undesirable mass splitting for \WR in comparison to \SER which shows a more meaningful interpolation. All images have been normalized to densities in $\R^2$ and we keep the same intensity range for a fair visualization. See \cref{fig:twos} for an additional example.} 
    \label{fig:umbrella}
\end{figure}

\subsubsection*{Results}

\Cref{fig:MNIST_multiple} and \cref{fig:doodle_multiple} show the results for the MNIST and QuickDraw datasets respectively. Barycenters involve around $6000$ images per class. All input images were of dimension $28 \times 28$ and for the $\SE$ interpolation, we used 16 orientations for lifting and we chose the metric parameters setting $\zeta = 1.5$.  We normalize the results so that they are all visualized as probability densities on $\R^2$. 

Expectedly, \WR gives a visually more coherent representation than the \LR and it does not suffer from the blurry artifacts present in the straightforward linear interpolation. 
However, we can see that \SER is even sharper than \WR and more representative of each class, yielding even more meaningful interpolations. This can be attributed to the lifting to the group and an-isotropic metrics that encourage mass movement over line structures.   

In \cref{fig:umbrella} we compare \WR and \SER more directly. We perform an interpolation of 4 images (at the corners) with linear weights. Again, we observe more meaningful interpolations for \SER in comparison to \WR.

\subsection{Barycentric Interpolation of Orientation Fields}
A natural application for optimal transport on $\SE$ is the interpolation of orientation fields on $\R^2$, where one can associate a probability measure on $\SE$ to an orientation field. As reported earlier, the interpolation of orientation fields using optimal transport has been explored in prior work, albeit on different geometries \cite{solomon2019optimal, peyre2019quantum}.  
However, as we explain below, orientation fields on Euclidean domains can be naturally lifted onto the corresponding roto-translation group. We can therefore obtain a more convenient extension of \cite{solomon2015convolutional} which allows for the efficient transport of dense orientation fields without explicitly needing to compute or store the OT coupling. 
\subsubsection*{Lifting Orientation Fields: Practical Aspects}
\label{sec:vector-fields}
Given a vector field on $\R^2$ expressed in polar coordinates $v(\cdot)=|v(\cdot)|\;(\cos \psi(\cdot),\sin \psi(\cdot))^T$, with $\int_{\R^2} |v(x,y)| {\rm d}x{\rm d}y=1$, we lift it towards a regular probability measure $U\in \mathcal{P}(SE(2))$ via density: 
\begin{equation} \label{equ:lifting_vector_field}
U(x,y,\theta) = \begin{cases}
    |v(x,y)|
    & \text{if } d_{S^1}(\psi(x,y),\theta)< \frac{\pi}{N}  \\
    0 & \textrm{else}
\end{cases} ,
\end{equation}
where $N$ denotes the number of orientations (sampled equidistantly at steps of $\frac{2\pi}{N}$). We project back by
\begin{equation}
\label{equ:reconstruction_vector_field}   
v(x,y)= 
\int_{-\pi}^{\pi} U(x,y,\theta) \; {\rm d}\theta \; \left(
\begin{array}{c}
\cos \psi^*(x,y) \\ \sin \psi^*(x,y)
\end{array}
\right) \ , \textrm{ with }\psi^*(x,y)=\underset{\theta \in [-\pi,\pi)}{\argmax}\;  U(x,y,\theta)
\end{equation}
This basic practical lifting replaces the Dirac distribution in \cref{equ:lifting_vector_field2} by uniform distributions over 0-th order B-splines. The construction \cref{equ:lifting_vector_field} and reconstruction \cref{equ:reconstruction_vector_field} can be generalized to advanced higher-order B-spline orientation channels \cite{Felsberg2} but this is beyond the scope of this article.    
\begin{figure}[t]
    \centering   
    \includegraphics[width=1\linewidth]{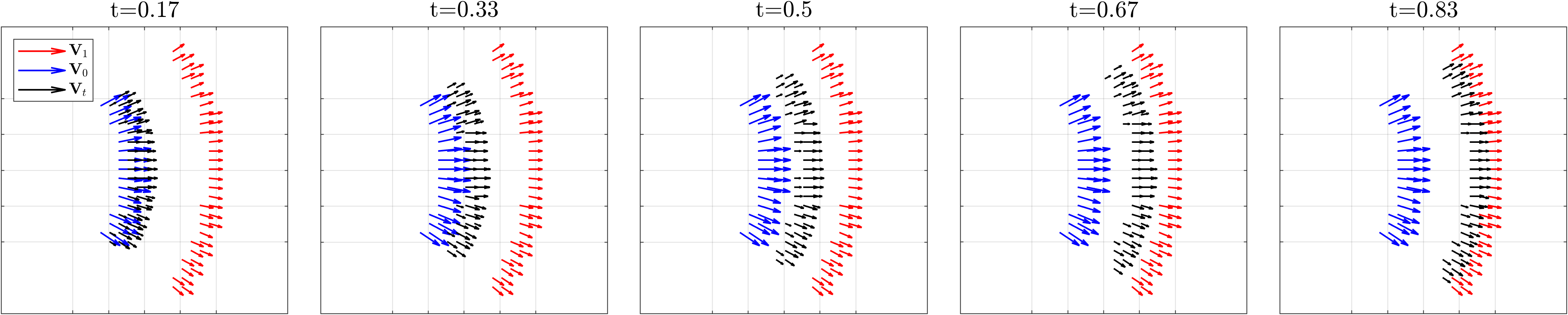}
    \caption{{\bf Entropic optimal transport with orientation fields on $\R^2$:} An example of \emph{barycentric transport} of orientation fields defined over $\mathbb{R}^2$. We lift an orientation field on $\R^2$ to the group $\SE$ by assigning a magnitude and orientation to each location $(x,y) \in \mathbb{R}^2$. Using our method, we then compute the 2-Wasserstein barycenter of the two lifted measures and project back on the plane to visualize the \emph{interpolated} orientation field. In this example we use a isotropic metric $(w_1,w_2,w_3) = (1, 1, 1)$.} 
    \label{fig:vf_interp}
\end{figure}
We can see an example of this for the localized vector fields displayed in blue and red in \cref{fig:vf_interp}. The black arrows are obtained by lifting the two vector fields, interpolating the resulting measure in $\SE$ and then projecting back to a vector field on $\R^2$. The result is a natural way of interpolating the two vector fields, both in location and direction. In \cref{fig:vf_car} we can see a comparison between a (spatially) isotropic metric, i.e. $\zeta = 1$ and an anisotropic metric, $\zeta \gg 1$. Depending on what the vector field represents, different spatial anisotropies can be used. For example, the highly anisotropic model can serve as a model for Reeds-Shepp car movement \cite{duits2018optimal}, modeling movement with certain curvature constraints.

This procedure can be extended to vector fields of any dimension $d$, by lifting them to the homogeneous space of positions and orientations $\mathbb{M}_d=SE(d)/SO(d\!-\!1)$, and performing optimal transport over these spaces. As mentioned previously, most of our developments carry over naturally to homogeneous spaces.
\begin{figure}[t]
    \centering   
    \includegraphics[width=1\linewidth]{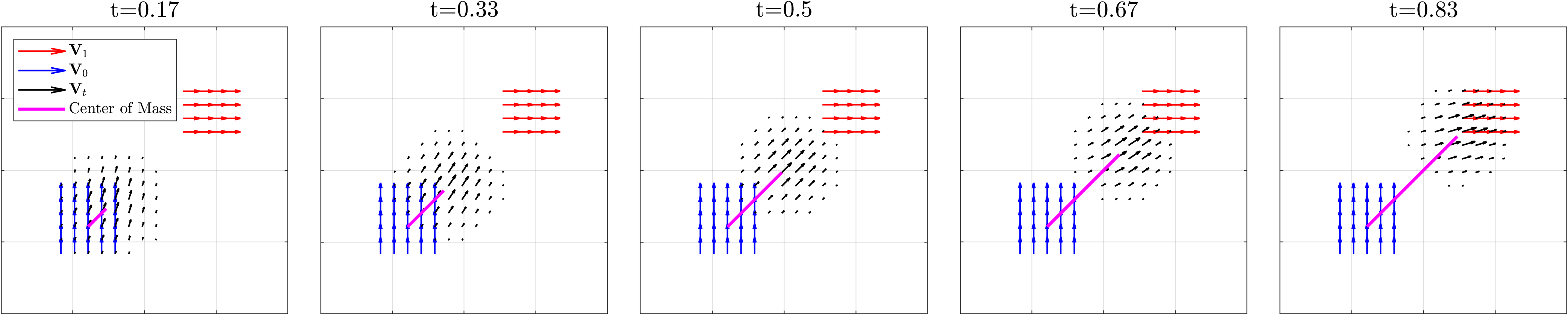}
    \includegraphics[width=1\linewidth]{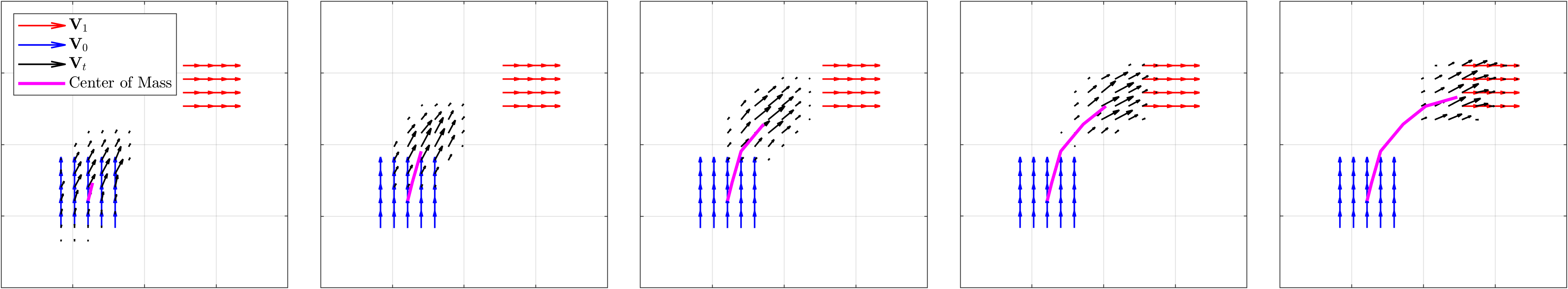}
    \caption{Another example similar to \cref{fig:vf_interp} where we interpolate a localized orientation field using different metrics in $\SE$. (Top) Isotropic: $(w_1,w_2,w_3) = (1, 1, 1)$ (Bottom) Anisotropic: $(w_1,w_2,w_3) = (1, 5, 1)$. We highlight how choosing anisotropic metrics in $\SE$ yields a natural u-curve  and serves as a model for Reeds-Shepp car movement \cite{duits2018optimal}} 
    \label{fig:vf_car}
\end{figure}
\subsection{Wasserstein Gradient Flows for Line Completion}
As a final example, we implement entropic Wasserstein gradient flows for solving PDEs. 
We use this setting to demonstrate a useful attribute of left-equivariant anisotropic $\SE$ image processing:  
\emph{line-completion}. 
We consider the images on the left of \cref{fig:diff_erosion_1} as our guiding example. This image comprises a dominant line structure placed amid randomly located smaller line elements.
We illustrate that we can complete and extract the main line by solving a non-linear diffusion equation in $\SE$. For this, we take such an image and view it as an initial condition $\mu_0\in \cP(G)$, with either \(G = \SE\) or \(\R^2\). We then compute the solution of 
\begin{equation}
\frac{\partial \mu}{\partial t} = \Delta_{\mathcal{G}} \mu^m
\label{eq:pde_ndiff}
\end{equation}
in a time-interval $[0, T]$, and with a left-invariant metric $\mathcal{G}$ on the Lie group \(G\). The PDE in \cref{eq:pde_ndiff} represents the porous medium equation and presents a Wasserstein gradient flow structure \cite{otto2001geometry} as per equation \cref{equ:cost_GF} with the functional
\begin{equation}
    F(\mu) = \frac{1}{m-1}\int_{G} \mu(g)^m \dhaar{g},
\end{equation}
for $m > 1$. 

We compute the solution for both groups $\R^2$ and $\SE$ with \(m = 5\) by solving the discrete gradient flow time-marching \cref{equ:GF_JKO} which we solve following the scheme introduced in \cite{peyre2015entropic}. 
On \(\R^2\) we use the standard Euclidean metric, and on \(\SE\) we choose anisotropic metric parameters \cref{eq:metric_SE2} to $(w_1,w_2,w_3) = (1, \sqrt{10}, \sqrt{2})$.

Diffusion in $\SE$ with such high anisotropic metrics have a \emph{line completion} behavior where dominant line structures get naturally connected in the evolution process. We can see this effect quite clearly in \cref{fig:diff_erosion_1}. We highlight that the anisotropy of $\SE$ leads to a strong \emph{local} orientation preference and this accumulates into the global line-connectivity behavior. In contrast, the isotropic diffusion in $\R^2$ does not yield the same result and the blur is uniform throughout the domain. Moreover, unlike $\SE$, left-invariant anisotropic metrics in $\R^2$ have a global preference for direction that does not encourage line connectivity in every situation. We include an additional, slightly more challenging example in \cref{fig:diff_erosion_2} to illustrate this point.

\Cref{fig:diff_erosion_1} and \cref{fig:diff_erosion_2} demonstrate the numerical implementation of the entropic Wasserstein gradient flow on $\SE$ using our distance approximation $\rho_b$ and thus serve as an additional practical validation of our framework.
\begin{figure}[t]
    \centering   
    \includegraphics[width=1\linewidth]{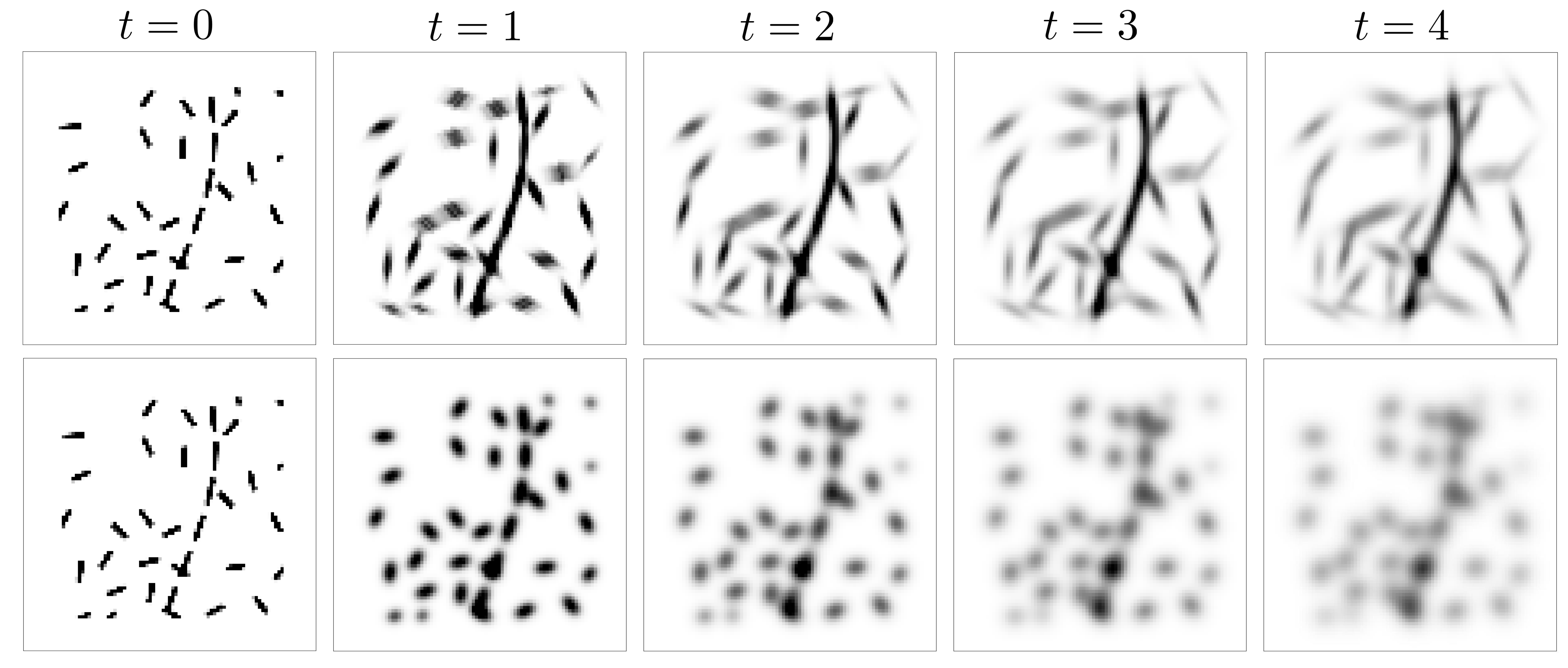}
    \vspace{-0.5cm}
    \caption{\textbf{Gradient flow in $\SE$} We implement a non-linear diffusion \cref{eq:pde_ndiff} with $m=5$ and with entropic Wasserstein gradient flow on both $\SE$ (top row) and $\R^2$ (bottom row) using the same initialization. Different from $\R^2$, diffusions with left-invariant anisotropic metrics in $\SE$ show a distinct \emph{line connectivity} behavior as depicted in the evolutions. See \cref{fig:diff_erosion_2} for an additional example.}    
    \label{fig:diff_erosion_1}
    
\end{figure}
\section{Conclusion}
In conclusion, we build and apply an efficient computational framework for optimal transportation over the roto-translation group $\SE$. We summarise the novel and impactful aspects of our contributions below:

\begin{enumerate}
\item We leverage distance approximations and group convolution operations toward efficiently extending the well-known Sinkhorn algorithm to $\SE$. We theoretically justify our construction with \cref{lem:w-bound} and \cref{lem:error_bound_optimizers}. We numerically validate our approximation by tracing sub-Riemannian geodesics in \cref{fig:geodesic_path} and \cref{fig:s_curve} and report their fidelity to the analytical solutions.
\item We show more sharper and meaningful barycentric interpolations of images in \cref{fig:MNIST_multiple}, \cref{fig:doodle_multiple} and \cref{fig:umbrella}. These results are aided by the line-completing behavior of left-invariant anisotropic metrics on $\SE$. This demonstrates the benefits of equivariant lifting-processing-projection of images onto $\SE$ with optimal transport.
\item In \cref{lem:equivariance} we derive the invariance and equivariance properties of optimal transport on $\SE$. We theoretically establish that for a wide class of Lie Groups, including $\SE$, the left/right group action is generally \underline{not} equal to the optimal transport map between a final distribution that is a group action of an initial distribution in \cref{prop:compact_subgroup}.
\item We show the Lie group $\SE$ allows for an unique Riemannian and sub-Riemannian interpolation of orientation fields on $\R^2$ by lifting and projecting them from $\SE$ in \cref{fig:vf_interp} and  \cref{fig:vf_car}.
\item We apply an entropic Wasserstein gradient flow on $\SE$ for the porous media equation. We show that the choice of left-invariant anisotropic metrics leads to the equivariant \emph{line completion} behavior of the diffusion in \cref{fig:diff_erosion_1}.
\end{enumerate}
As exciting future work, we aim to extend to other Lie groups of practical significance like SO(3) and SE(3). In addition, developing accurate and efficient lifting and projection operations for Lie groups is also an avenue for prospective research in many practical applications where optimal transport could be applied. 

\section*{Acknowledgments}
We gratefully acknowledge the Dutch Foundation of Science NWO for its financial support by Talent Programme VICI 2020 Exact Sciences (Duits, Geometric learning for Image Analysis, VI.C. 202-031). Olga Mula gratefully acknowledges support from the DesCartes project.

\bibliographystyle{siamplain}
\bibliography{references}

\appendix
\clearpage

\section{$\SE$ Wasserstein Barycenter Algorithm}
We enumerate the $\SE$ Wasserstein barycenter algorithm below. This is essentially identical to Algorithm 2 from \cite{solomon2015convolutional} (and the scaling algorithm from \cite{chizat2018scaling} specified to the barycenter problem), but with necessary replacements (\cref{sec:ER_scaling}) for the Lie Group $\SE$, i.e.~using group convolutions and local distance approximations $\rho_b$.  
\vspace{0.5cm}
\label{app:algorithm}
\begin{algorithmic}
\BeginBox
\Function{ \SE Wasserstein-Barycenter}{$ \textrm{Inputs: } \{\mu_i \in \cP(G)\}_{i = 1}^{n} ; \textrm{Weights: } \{\lambda_i\}_{i = 1}^{n} \in \Sigma_n ; \textrm{Regularization Parameter: } \epsilon$}

\LComment{Notation}
\State $* \gets \SE\textrm{ group convolution (ref. \cref{subsec:gconv})}$ 
\State $\otimes \gets \textrm{pointwise multiplication; }$ $\ominus \gets \textrm{pointwise division}$

\LComment{Initialization}

\State $v_1, v_2,  \dotsm v_n \do \gets \mathbf{1}$ 
\State $w_1, w_2,  \dotsm w_n \do \gets \mathbf{1}$
\State $K_{\epsilon} \gets e^{-\frac{\rho_b^2}{\epsilon}}$ using \cref{eq:rhob}

\LComment{Sinkhorn Iterations with Iterated Bregman Projections}
\For{j = 1 \dots }

    \State $\mu \gets \mathbf{1}$

    \State \For{i = 1 \dots n}
    \State $w_i \gets \mu_i \ominus \{ K_{\epsilon} * v_i\}$ 
    \State $d_i \gets v_i \otimes \{ K_{\epsilon} * w_i\}$
    \State $\mu \gets \mu \otimes d_i^{\lambda_i}$
    \EndFor

    \State \For{i = 1 \dots n}

        \State $ v_i \gets v_i \otimes \mu \ominus d_i $ 
    
    \EndFor
    
\EndFor
\State \Return $\mu$
\EndFunction
\EndBox
\end{algorithmic}
\clearpage
\section{Invertibility of the Orientation Score Transform}
\label{app:reconstruct}
In this appendix, we summarise some important aspects of the orientation score transform, especially concerning its invertibility and the reconstruction of disk-limited functions. We refer to \cite{duits2007image, duits2010left} for a formal derivation and \cite{remcolecture} for an intuitive enumeration of all these results. 
The choice of wavelet function $\psi \in L^{2}(\mathbb{R}^2) \cap L^{1}(\mathbb{R}^2)$ affects the stability of the  orientation score transform $\mathcal{W}_{\psi}: L^{2,\varrho}(\mathbb{R}^2) \to L^{2}(SE(2))$. In practice, choosing cake-wavelets enables a stable, straightforward projection \cref{eq:proj} to reconstruct disk-limited images: 
\[
f \in L^{2,\varrho}(\mathbb{R}^2):= \{f \in L^{2}(\mathbb{R}^2)\;|\; \textrm{supp}(\mathcal{F}f) \subset B_{\mathbf{0},\varrho}\}
\]
where $\mathcal{F}$ denotes the  Fourier transform on $L^{2}(\mathbb{R}^2)$ and where
$B_{\mathbf{0},\varrho}=\{\omega \in \mathbb{R}^2\;|\; \|\omega\|<\varrho\}$. The value $\varrho$ is a priori fixed and in practice chosen close to the Nyquist frequency due to sampling. 
\\
{\bf Invertibility and Unitarity:} The orientation score transform is invertible and unitary (therefore also an isometry) \emph{but} with respect to a \emph{Sobolev} type of norm in a specific reproducing kernel Hilbert space defined on \SE. More specifically, the range of the orientation score transform is the unique reproducing kernel Hilbert space $\mathbb{C}_{K}^{\SE}$ consisting of complex-valued functions on $\SE$ with the reproducing kernel given by $K(g,h)=(\mathcal{U}_{g}\psi,\mathcal{U}_{h}\psi)$. As such we can define $W_{\psi}: L^{2,\varrho}(\mathbb{R}^2)\to\mathbb{C}_{K}^{\SE}$ again by  $(W_{\psi}f)(g)=(\mathcal{U}_{g}\psi,f)$. Then it turns out \cite[Thm.2]{remcolecture} that the the reproducing kernel norm  simplifies to 
\begin{equation}\label{eq:unitarity}
(U,V)_{M_{\psi}}=
\int \limits_{-\pi}^{\pi}
\int \limits_{\mathbb{R}^2} 
\overline{
\hat{U}(\omega,\theta)}
\hat{V}(\omega,\theta)
|M_{\psi}(\omega)|^{-1}{\rm d}\omega{\rm d}\theta 
\implies
\|f\|_{L^2}= \|\mathcal{W}_{\psi}f\|_{M_{\psi}}
\end{equation}
where $\hat{U}(\omega,\theta)=(\mathcal{F}U(\cdot, \theta))(\omega)$. The continuous isotropic function $M_{\psi}$ is given by:
\[
M_{\psi}(\omega)=
2\pi \int_{-\pi}^{\pi} 
|\hat{\psi}(R_{\theta}^{-1}\omega)|^2\; {\rm d}\theta,
\]
where $\hat{\psi}=\mathcal{F}\psi$. Intuitively $M_{\psi}(\omega)$ controls how well all the group coherent wavelets together fill up in the Fourier domain. Since $W_{\psi}$ is unitary we have an \emph{exact} image reconstruction \cite[eq.2.8]{duits2010left} given by:
\begin{equation}\label{eq:exactreco}
f= W_{\psi}^* W_{\psi} f =
\int \limits_{-\pi}^{\pi} 
\mathcal{F}^{-1}
\left[
\widehat{
W_{\psi}f}(\cdot,\theta)
 \; \frac{\hat{\psi}(R_{\theta}^{-1}
\cdot)}{M_{\psi}(\cdot)} \;  \right] {\rm d}\theta.
\end{equation}
Since we are interested in $f \in L^{2,\varrho}(\R^2)$, we must ensure that $\psi$ is chosen such that 
\begin{equation} \label{eq:stability}
\delta \leq M_{\psi}(\omega) \leq M , \textrm{for all }\omega \in B_{\mathbf{0},\varrho}, \textrm{ for some }\delta, M>0.
\end{equation}
{\bf Stability:} 
When $\delta$ and $M$ are chosen to be respectively the maximum and minimum of the continuous function $M_{\psi}$ constrained to the closed ball
$\|\omega\| \leq \varrho$ then the condition number (from $L^2$ to $L^2$) of the original orientation score transform
$\mathcal{W}_{\psi}$ equals $M/\delta$.
Wavelets that satisfy (\ref{eq:stability}) and guarantee a good condition number are called \emph{proper wavelets}, cf.~ \cite[Thm.3, Def. 15]{remcolecture}.
Proper wavelets do not tamper with data evidence and allow for stable (exact) reconstruction. 
\\
{\bf Projection/Fast Reconstruction Property:} 
Additionally, some proper wavelets also have a fast (approximative) reconstruction property. This means that 
\begin{equation}
\label{eq:fastreco}
\int_{-\pi}^{\pi}
\hat{\psi}_{\theta} \; {\rm d}\theta \approx 1_{B_{\mathbf{0},\varrho}}  \Leftrightarrow 
\forall_{f \in L^{2,\varrho}(\mathbb{R}^2)} \;:\; 
\int_{-\pi}^{\pi} 
\mathcal{W}_{\psi}f(\cdot,\theta)\; {\rm d}\theta
=
\int_{-\pi}^{\pi} \psi_{\theta} \;  {\rm d}\theta *f \approx f 
\end{equation}
with $\psi_{\theta}(\mathbf{x})=\psi(R_{\theta}^{-1} \mathbf{x})$,
for details see \cite[Def.15]{remcolecture}. 
\\
{\bf Cake-wavelets:} 
Cake-wavelets satisfy all the abovementioned properties: Unitarity \cref{eq:unitarity}, Stability \cref{eq:stability}, and Fast Reconstruction \cref{eq:fastreco}. 
In the Fourier domain, cake-wavelets smoothly approximate indicator functions supported on cones allowing us to omit the extra term in (\ref{eq:exactreco}) in order to obtain a stable and accurate reconstruction \cref{eq:fastreco}. We refer to \cite[ch.2.1.5]{bekkers2017retinal} for exact formulas and intuitive illustrations of the construction of cake-wavelets and \cite{duits2010left,vandenbergJMIV2023,bekkers2014multi,baspinar2021cortical, bertalmio2021cortical} for their different applications in practice. 

\section{Idealized Lifts of Orientation Fields in $\R^2$ to $\SE$
\label{subsec:lift_of}}
Identify $S^1 \equiv \R/(2\pi \mathbb{Z}) \equiv [0,2\pi)$.
Suppose we have an \textit{orientation field} \(\theta : \Omega \to [0,2\pi)\) with Lebesgue measure $\mu(\Omega)$ nonzero and finite. 
We can lift such an orientation field to a probability measure $U \in \cP(\SE)$ given by:
\begin{equation} \label{equ:lifting_vector_field2}
    U(x,y, \cdot) =
    \left\{
    \begin{array}{ll}
    \frac{1}{\mu(\Omega)} \delta_{\theta(x,y)}(\cdot) &\textrm{if }(x,y) \in \Omega \\
    0 & \textrm{ else }
    \end{array}
    \right.
\end{equation}
where we put a uniform distribution on $\Omega$ and at each position $(x,y) \in \Omega$ a Dirac measure w.r.t. angular variable $\theta$ centered at $\theta(x,y)$. 
Equation \cref{equ:lifting_vector_field2} can be interpreted as an \emph{idealized} lift. 
One can reconstruct the orientation field
$\theta: \Omega \to [0,2\pi)$
from the support of $U$ by
\begin{equation} 
    \theta(x,y) = 
        \theta \text{ if } (x,y,\theta) \in \text{supp } U \textrm{ for all }(x,y) \in \Omega.
\end{equation}
In practice, however, we discretize both in space and orientation, and the intermediate probability measures obtained through the optimal transport have a support that is spread out over $\theta$ in general.
This means this lift and projection above can \textit{not} be used unadapted in our application.
We explain a practical approximation in \cref{sec:vector-fields}  that comfortably addresses this issue and we show meaningful interpolation of orientation fields even with non-uniform spatial distributions (cf. equation \cref{equ:lifting_vector_field}). This approach could easily be generalized to B-spline orientation channels of a higher degree \cite{Felsberg2} than 0.
\begin{figure}[h]
    \centering   
    \subfloat{\includegraphics[width=1\linewidth]{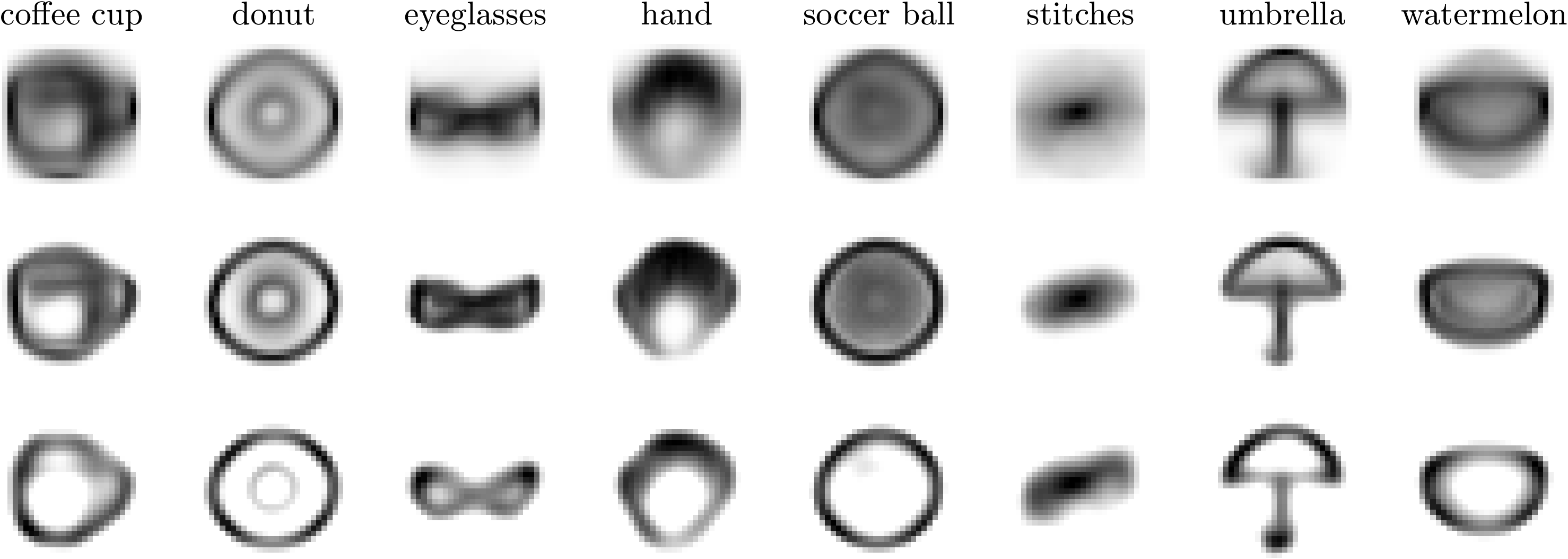}}
    \caption{Average barycenters of ($ \sim 6000$ images per class) images from the QuickDraw dataset. {\bf 1\textsuperscript{st} Row}: $L^2$ barycenter of all images per class - \LR. {\bf 2\textsuperscript{nd} Row}: 2-Wasserstein barycenter in $\mathbb{R}^2$ - \WR, {\bf 3\textsuperscript{rd} Row}: 2-Wasserstein barycenter in $\SE$ - \SER. \SER shows sharper interpolations in comparison with its counterparts on $\mathbb{R}^2$.} 
    \label{fig:doodle_multiple_2}
\end{figure}
\section{Additional Examples}
\label{sec:additional_results}
We report some additional numerical examples using our method. In \cref{fig:doodle_multiple_2} we show additional barycenters from the Quickdraw\footnote{\url{https://quickdraw.withgoogle.com/data}} dataset and compare with \LR, \WR and \SER. Then, in \cref{fig:twos} we show another example similar to \cref{fig:umbrella} where we compare the barycenter interpolation between \WR and \SER on MNIST. Finally in \cref{fig:diff_erosion_2} we show an additional entropic Wasserstein gradient flow demonstration with a slightly more challenging line completion similar to \cref{fig:diff_erosion_1}. 
\begin{figure}[t]
    \centering   
    \includegraphics[width=1\linewidth, page=1]{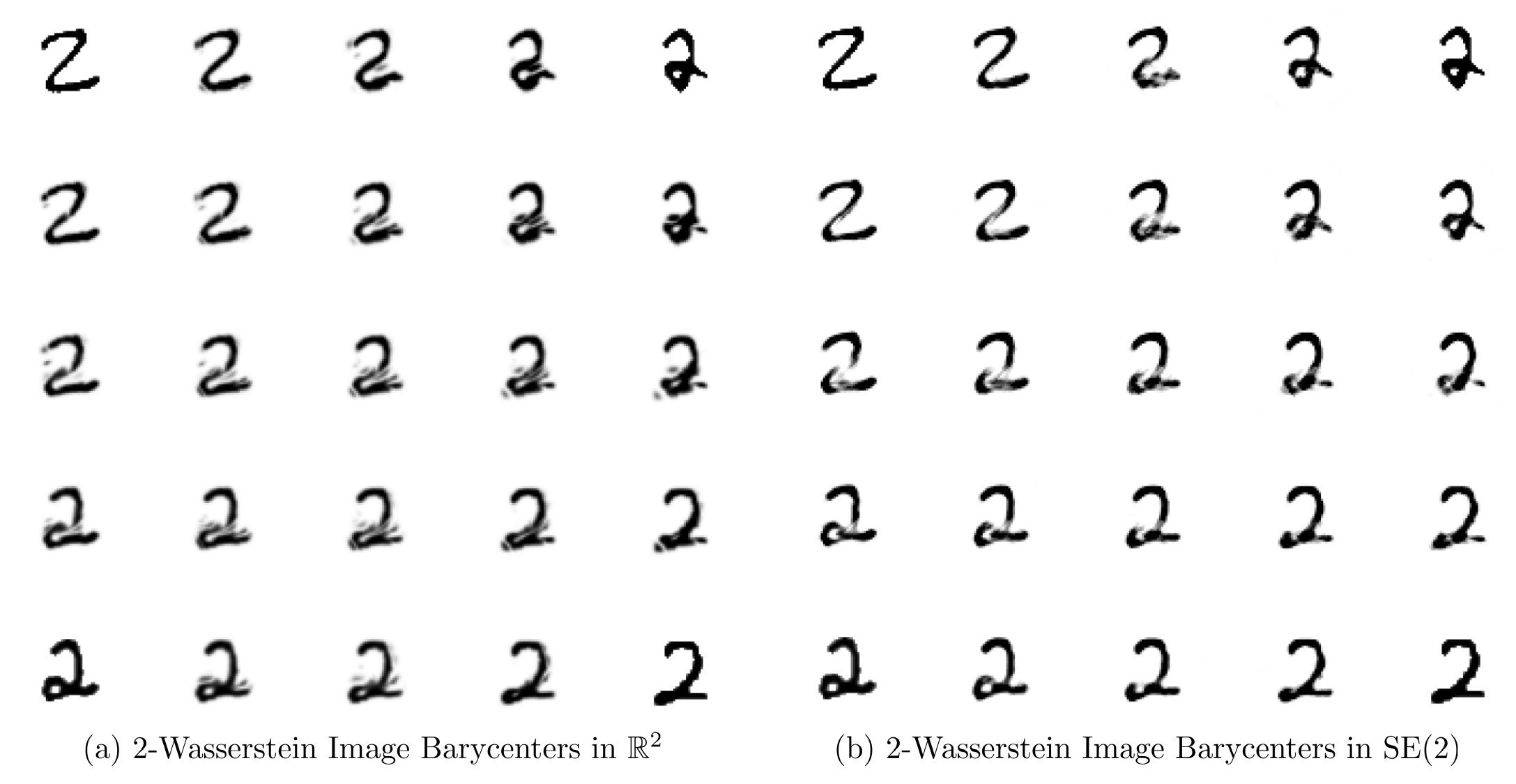}
    \caption{\textbf{Comparing image barycenters in $\R^2$ and $\SE$:} Another example similar to \cref{fig:umbrella}, with MNIST digits. The Images at the corners are interpolated using linear weights. Again, one observes a more meaningful interpolation with left-invariant anisotropic optimal transport in $\SE$} 
    \label{fig:twos}
    \vspace{0.5cm}
\end{figure}
\begin{figure}[t]
    \centering   
    \includegraphics[width=1\linewidth]{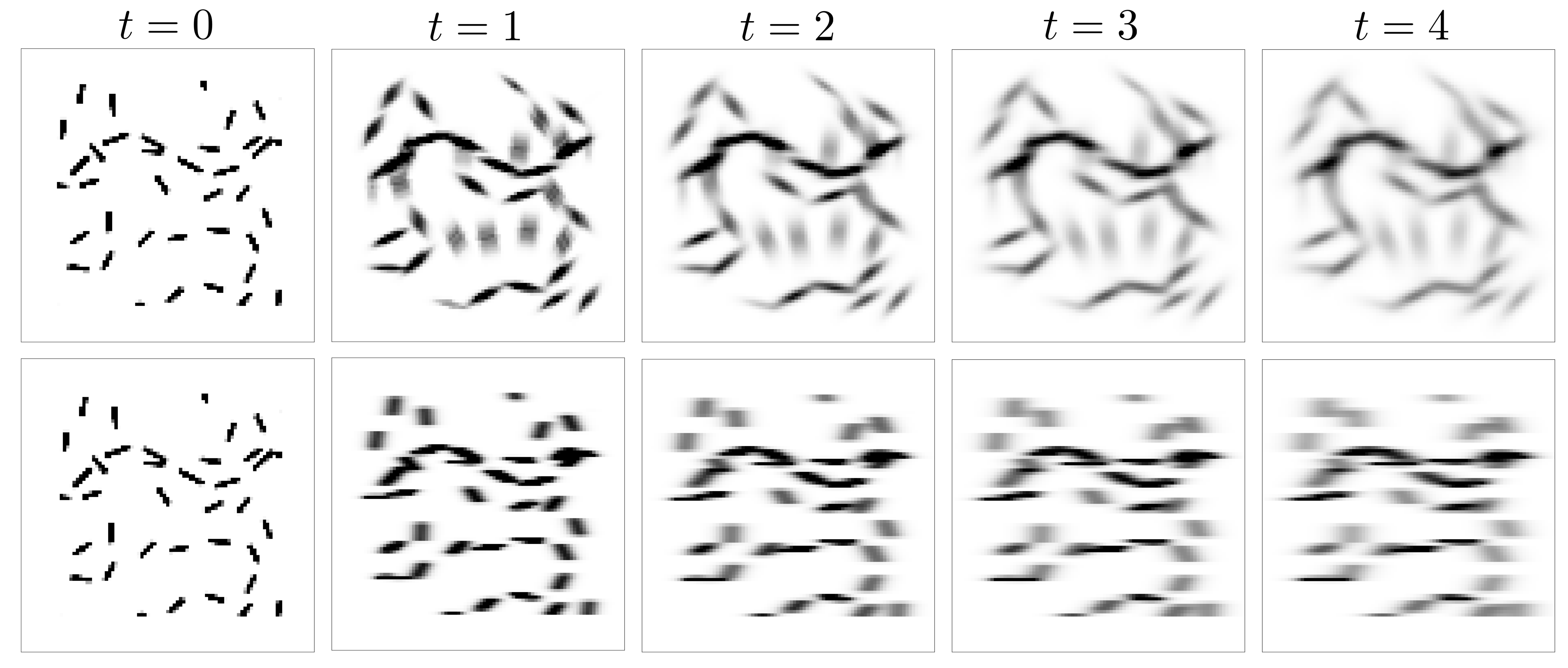}
    \caption{\textbf{Gradient flow in $\SE$} Another example similar to \cref{fig:diff_erosion_1} where we implement a entropic Wasserstein gradient flow for the PDE \cref{eq:pde_ndiff} on both $\SE$ (top row) and $\R^2$ (bottom row) using the same initialization. Here, we apply an \emph{anisotropic} diffusion on $\SE$ \underline{and} $\R^2$. Note that OT with this metric on $\R^2$ is no longer equivariant to rotations. The rototranslation equivariant $\SE$ flow (top row) shows a more distinct line completion behavior than the $\R^2$ flow (bottom) which adheres to a global directional preference.} 
    \label{fig:diff_erosion_2}
\end{figure}
\end{document}